\documentclass[journal]{IEEEtran}

\usepackage{amsmath,amssymb,graphicx}
\usepackage{cite}

\newtheorem{theorem}{Theorem}
\newtheorem{corollary}{Corollary}
\newtheorem{assumption}{Assumption}
\newtheorem{remark}{Remark}
\newenvironment{proof}[1][Proof]{\par\noindent\textit{#1.}\ }{\hfill$\square$\par}

\newcommand{\thetaml}{\widehat{\theta}_{\mathrm{ML}}}
\newcommand{\thetaeb}{\widehat{\theta}_{\mathrm{EB}}}
\newcommand{\thetamix}{\widehat{\theta}_{\mathrm{mix}}}
\newcommand{\xmse}{\mathrm{XMSE}}
\newcommand{\E}{\mathbb{E}}
\newcommand{\Tr}{\mathrm{Tr}}

\title{XMSE-Aware Adaptive Empirical Bayes Estimation}
\author{Minghao~Chen and Jiale~Zheng%
\thanks{M. Chen is with Tencent Technology, Shenzhen, China
(e-mail: monychen@tencent.com).}%
\thanks{J. Zheng is with Huawei Noah's Ark Lab, Shenzhen, China
(e-mail: zhengjiale2@huawei.com).}}

\begin{document}
\maketitle

\begin{abstract}
Empirical Bayes (EB) estimators can match the first-order asymptotic risk of
maximum likelihood (ML) while behaving very differently at second order:
recent excess mean squared error (XMSE) analysis shows that kernel-based EB
estimation may be worse than ML when the kernel is poorly aligned with the
true parameter.  This paper turns that diagnostic into a design principle.
We propose an XMSE-aware mixed estimator that interpolates between ML and EB
shrinkage.  Its fixed-weight XMSE is a scalar quadratic, yielding a
closed-form oracle mixing weight that is no worse than both ML and the base
EB estimator at the XMSE scale.  A plug-in implementation based on
finite-sample XMSE approximations is proved consistent, with a second-order
oracle regret rate for an interior oracle weight.  We further establish a
transfer of the regret bound to the fixed-weight risk curve evaluated at the
selected weight, a thresholded boundary rule, and extensions to compact kernel
families and to finite and growing kernel dictionaries with high-probability
oracle bounds.
Finite impulse response simulations with SURE-tuned, hard-selection, and
trace-corrected baselines, together with the public Silverbox and Cascaded
Tanks benchmarks, show that the proposed estimator retains most of the
benefit of regularization when it is helpful and retreats toward ML under
kernel misspecification, with an identified finite-sample calibration
failure mode analyzed on the benchmarks.
\end{abstract}

\begin{IEEEkeywords}
Empirical Bayes, excess mean squared error, regularized system
identification, hyperparameter estimation, kernel selection.
\end{IEEEkeywords}

\section{Introduction}

Empirical Bayes (EB) and kernel-based regularized estimators are widely used in
linear regression and system identification \cite{efron1972,maritz2018,
pillonetto2022}.  Their first-order asymptotic MSE often coincides with that
of the maximum likelihood (ML) estimator, so first-order theory cannot explain
when regularization helps or hurts.  The XMSE framework addresses this
limitation by studying the second-order difference
\[
  \xmse(\widehat{\theta})
  =
  \lim_{N\to\infty} N^2
  \left[
    \mathrm{MSE}(\widehat{\theta})
    -
    \mathrm{MSE}(\thetaml)
  \right].
\]
Existing XMSE analysis is primarily diagnostic: it decomposes the EB risk gap
and identifies kernel-parameter misalignment as a source of positive XMSE.
This paper turns the diagnostic into a design principle.

The main contributions are as follows.
\begin{enumerate}
  \item We introduce an XMSE-aware mixed estimator that interpolates between
  ML and a base EB estimator.
  \item We derive the fixed-weight XMSE formula and show that its oracle
  mixing weight has a closed form, including explicit improvement regimes that
  explain when the oracle rule selects ML, a strict mixture, or EB.
  \item We prove that the oracle mixed estimator is no worse than both ML and
  the base EB estimator at the XMSE scale.
  \item We develop a plug-in implementation based on finite-sample XMSE
  component estimates and establish consistency, a boundary-robust plug-in
  oracle inequality, and a second-order oracle regret bound for the selected
  weight at an interior oracle optimum.  We further prove that the
  inherited expansion holds uniformly in the mixing weight, so the regret
  bound transfers from the XMSE criterion to the deterministic fixed-weight
  risk curve evaluated at the selected weight; the corresponding expansion for
  the randomly weighted estimator itself is posed as an open problem.  We also
  give a thresholded zero-bias boundary rule.
  For the scaled-kernel case, we provide primitive sufficient conditions for
  consistency of the plug-in XMSE components, strengthen these conditions to
  compact scaled-kernel families by uniform convergence, and extend the same
  logic to compact multi-parameter kernel families under uniform criterion
  convergence.
  \item We show that the same plug-in criterion yields a finite-set oracle
  inequality and, under a positive oracle gap, a consistent selector over a
  finite candidate set of kernels.  The same argument extends to growing
  candidate dictionaries under a uniform plug-in approximation condition, with
  a high-probability oracle bound and a sub-Gaussian rate specialization.
  \item We demonstrate, in FIR system-identification simulations that include
  SURE-tuned, hard-selection, trace-corrected, and marginal-likelihood
  baselines, that the proposed estimator protects against kernel mismatch
  while retaining most of the benefit of EB regularization under
  better-aligned choices.  On the public Silverbox and Cascaded Tanks
  benchmarks we report both the intended retreat behavior and an identified
  finite-sample calibration failure mode of the plug-in components.
\end{enumerate}

\section{Related Work}

The empirical Bayes view of shrinkage estimation goes back to Robbins
\cite{robbins1956} and to the classical James--Stein and empirical Bayes
literature on risk reduction \cite{james1961,efron1972,efron1973}.  Modern
parametric empirical Bayes methods use data-dependent hyperparameters to tune
prior or penalty structure \cite{morris1983,petrone2014}.  SURE-type risk
estimation traces to Stein's unbiased risk estimate \cite{stein1981}, while
GCV and kernel smoothing have a long connection to spline and Gaussian-process
methods \cite{wahba1990,rasmussen2006}.  In system identification, regularized
FIR estimation builds on the classical prediction-error framework
\cite{ljung1999}, and kernel-based regularization has become a standard way to
encode smoothness and decay of impulse responses
\cite{pillonetto2010,chen2012,pillonetto2014,chen2014ssm,chen2014sparse,
chen2015obf,pillonetto2016atomic,carli2017maxent,chen2018kernel,chenmh2020,
chenmh2022,chiuso2016,
pillonetto2022}.  Hyperparameters are often selected by marginal likelihood,
SURE, or GCV, with scalable implementation and robustness issues studied in
\cite{chen2013implementation,pillonetto2015tuning}; their asymptotic properties
have been studied in
\cite{mu2018,mu2018gcv,ju2021gcv,ju2022sure,mu2024cv}, and a unified family of
hyperparameter estimators linking EB and SURE has been proposed in
\cite{zhang2024family}.

The present paper is closest to the recent XMSE analysis of empirical Bayes
estimators \cite{ju2026xmse}, which shows that first-order asymptotic risk is
insufficient for distinguishing EB and ML estimators and that kernel
misalignment can make EB worse than ML.  A related line gives deterministic
conditions under which kernel-based regularization cannot improve on the
least-squares (ML) estimate \cite{mu2024when}; rather than certifying such cases
in advance, our rule detects harmful regularization from data at the
second-order risk scale and retreats toward ML.  Our contribution is
complementary:
instead of using XMSE only as an explanatory quantity, we use the same
second-order risk scale to design a mixed estimator, derive its oracle rule, and
analyze the plug-in implementation.  The fixed-weight quadratic criterion,
closed-form projected oracle weight, plug-in oracle inequality, XMSE-transfer
result, thresholded boundary rule, and finite/growing dictionary selectors are
therefore decision rules and guarantees built on top of the reference
decomposition, not direct restatements of the diagnostic XMSE formula.
Another recent XMSE-based direction constructs generalized Bayes and
closed-form biased estimators that match the XMSE of an EB regularized
estimator while avoiding hyperparameter estimation \cite{ju2025biased}.  That
line replaces the EB implementation by computationally simpler estimators with
comparable second-order behavior.  In contrast, the present paper keeps the
base EB estimator available and uses XMSE to decide how much to mix it with ML,
including plug-in oracle-regret, XMSE-transfer, boundary, and kernel-selection
guarantees.

\section{Problem Setup}

Consider the finite impulse response model
\[
  Y = \Phi \theta_0 + E,
  \qquad
  E \sim \mathcal{N}(0,\sigma^2 I_N),
\]
where $\Phi\in\mathbb{R}^{N\times n}$ is deterministic with full column rank
and
\[
  \frac{1}{N}\Phi^\top\Phi \to \Sigma \succ 0.
\]
Let $\thetaml = (\Phi^\top\Phi)^{-1}\Phi^\top Y$ be the ML estimator.
Throughout the paper, the risk of an estimator $\widehat{\theta}$ is the
parameter-space mean squared error
\[
  \mathrm{MSE}(\widehat{\theta})
  =
  \E\,\|\widehat{\theta}-\theta_0\|_2^2,
\]
where the expectation is over the noise $E$, and
$\Pi_{[0,1]}(x)=\min\{1,\max\{0,x\}\}$ denotes the projection onto $[0,1]$.
In this Gaussian model the ML estimator is exactly unbiased,
$\E[\thetaml]=\theta_0$, with covariance $\sigma^2(\Phi^\top\Phi)^{-1}$.

Let
\[
  \thetaeb = \widehat{\theta}_{\mathrm{EB}}(\widehat{\eta}(\thetaml))
\]
be an EB estimator whose hyperparameter estimate $\widehat{\eta}$ is a
function of $\thetaml$, as in the reference XMSE analysis \cite{ju2026xmse}.
Assumption~\ref{ass:inherited} below states, in self-contained moment form,
the minimal part of the reference XMSE expansion
\[
  \xmse(\thetaeb)
  =
  B_{\mathrm{EB}} + V_{\mathrm{EB}} + H_{\mathrm{EB}}
\]
needed for the mixed estimator, where
$B_{\mathrm{EB}}$, $V_{\mathrm{EB}}$, and $H_{\mathrm{EB}}$ are,
respectively, the second-order squared-bias, variance, and
hyperparameter-estimation components.

\section{XMSE-Aware Mixed Estimator}

For a mixing weight $\omega\in[0,1]$, define
\[
  \thetamix(\omega)
  =
  \thetaml + \omega(\thetaeb-\thetaml).
\]
The endpoints recover ML and EB:
\[
  \thetamix(0)=\thetaml,\qquad
  \thetamix(1)=\thetaeb.
\]

\section{Main Results}

The results below are organized by the strength of the assumption needed.  The
fixed-weight and oracle results use only the inherited XMSE expansion for the
base EB estimator.  The plug-in weight results add consistency of the estimated
XMSE components.  The compact scaled-kernel corollary gives a uniform
continuous-mapping route for practical hyperparameter searches, while the
compact multi-parameter result replaces closed-form scale consistency with
uniform criterion and component convergence over a compact kernel-parameter
set.  The XMSE-transfer theorem evaluates the deterministic fixed-weight risk
curve at the selected weight; its key step, uniform convergence of that curve
over the weight interval, is proved from the inherited expansion rather than
assumed.  The kernel-selection results replace scalar component
consistency by uniform approximation of the minimized criterion over the candidate
dictionary.  This separation is important: the paper does not claim a
general nonsmooth random-weight XMSE expansion---the risk of the estimator
with the data-dependent weight inside the expectation is explicitly left
open---and the boundary and dictionary results are stated under the weaker
conditions needed for their respective conclusions.

For orientation, the theory proceeds in six layers.  The fixed-weight result
uses only the inherited EB XMSE expansion and gives the exact quadratic
criterion for $\thetamix(\omega)$.  The oracle result minimizes this
deterministic criterion over $[0,1]$ and gives dominance over both endpoints.
The plug-in results add consistent component estimates and give consistency,
an oracle inequality, and an interior second-order regret rate.  The compact
scaled-kernel corollary supplies a uniform route for continuous searches over
scaled kernel families, and the compact multi-parameter result supplies a
uniform argmin route for more general continuous kernel families.  The
XMSE-transfer result shows that the fixed-weight risk curve converges to the
quadratic criterion uniformly in the weight, so the criterion regret bound
carries over to that curve at the selected weight.  Finally, the
kernel-dictionary results replace scalar component consistency by uniform
plug-in approximation over the candidate set.

\begin{assumption}[Inherited XMSE expansion]
\label{ass:inherited}
Let $d_N=\thetaeb-\thetaml$ and split it along the hyperparameter map, as in
the reference decomposition \cite{ju2026xmse}, into
\[
  d_N = d_{1,N}+d_{2,N},
\]
with the fixed-hyperparameter part and the hyperparameter-estimation part
\[
\begin{aligned}
  d_{1,N}&=\widehat{\theta}_{\mathrm{EB}}(\widehat{\eta}(\theta_0))-\thetaml,\\
  d_{2,N}&=\widehat{\theta}_{\mathrm{EB}}(\widehat{\eta}(\thetaml))
    -\widehat{\theta}_{\mathrm{EB}}(\widehat{\eta}(\theta_0)).
\end{aligned}
\]
All moments below exist for every $N$, and the following limits hold:
\begin{enumerate}
  \item[(i)] (bias) $N\,\E[d_N]\to b_{\mathrm{EB}}$, and
  $B_{\mathrm{EB}}=\|b_{\mathrm{EB}}\|_2^2$;
  \item[(ii)] (variance component)
  $2N^2\,\E\{(\thetaml-\theta_0)^\top d_{1,N}\}\to V_{\mathrm{EB}}$;
  \item[(iii)] (hyperparameter-estimation component)
  $2N^2\,\E\{(\thetaml-\theta_0)^\top d_{2,N}\}\to H_{\mathrm{EB}}$;
  \item[(iv)] (higher-order term)
  $N^2\,\E\,\|d_N-\E[d_N]\|_2^2\to 0$.
\end{enumerate}
\end{assumption}

\begin{remark}[Relation to the reference decomposition]
\label{rem:assumption-source}
Conditions (i)--(iv) are a self-contained moment restatement of the
limits established in \cite{ju2026xmse} for kernel-based EB estimators with
hyperparameter estimators that are functions of $\thetaml$:
(i) is the second-order bias limit, (ii) and (iii) are the limits of the
variance and hyperparameter-estimation cross terms in the reference
decomposition, and (iv) is the vanishing $N^2$ limit of the centered
higher-order term.  Since $\E[\thetaml]=\theta_0$ and
$\E\|d_N\|_2^2=\|\E[d_N]\|_2^2+\E\|d_N-\E[d_N]\|_2^2$, conditions (i)--(iv)
immediately reproduce the reference expansion
$\xmse(\thetaeb)=B_{\mathrm{EB}}+V_{\mathrm{EB}}+H_{\mathrm{EB}}$.  For
orientation on the orders involved: in the scaled-kernel case
$d_{1,N}=O_p(N^{-1})$ with mean of exact order $N^{-1}$, while its centered
part and $d_{2,N}$'s centered part are $O_p(N^{-3/2})$, so (iv) reflects the
fact that the EB correction concentrates around its mean faster than the
$N^{-1}$ bias scale.
\end{remark}

\begin{theorem}[Fixed-weight XMSE]
\label{thm:fixed-weight}
Under Assumption~\ref{ass:inherited}, for any fixed $\omega\in[0,1]$,
\[
  \xmse(\thetamix(\omega))
  =
  \omega^2 B_{\mathrm{EB}}
  +
  \omega V_{\mathrm{EB}}
  +
  \omega H_{\mathrm{EB}}.
\]
\end{theorem}

\begin{proof}
Since $\thetamix(\omega)-\thetaml=\omega d_N$ with deterministic $\omega$,
expanding the square gives the exact identity
\[
\begin{aligned}
&\mathrm{MSE}(\thetamix(\omega))-\mathrm{MSE}(\thetaml)\\
&\quad =
2\omega\,\E\{(\thetaml-\theta_0)^\top d_N\}
+\omega^2\,\E\,\|d_N\|_2^2.
\end{aligned}
\]
Multiply by $N^2$ and take limits term by term.  By (ii) and (iii) of
Assumption~\ref{ass:inherited}, the cross term converges to
$\omega(V_{\mathrm{EB}}+H_{\mathrm{EB}})$.  For the quadratic term, write
$\E\|d_N\|_2^2=\|\E[d_N]\|_2^2+\E\|d_N-\E[d_N]\|_2^2$; by (i) the first part
contributes $\omega^2 B_{\mathrm{EB}}$, and by (iv) the second part vanishes
in the $N^2$ limit.
\end{proof}

\begin{theorem}[Oracle XMSE weight]
\label{thm:oracle-weight}
Let
\[
  C_{\mathrm{EB}} = V_{\mathrm{EB}}+H_{\mathrm{EB}}.
\]
An XMSE-minimizing fixed mixing weight over $[0,1]$ is
\[
  \omega^\star =
  \begin{cases}
  \Pi_{[0,1]}\left(-\dfrac{C_{\mathrm{EB}}}{2B_{\mathrm{EB}}}\right),
    & B_{\mathrm{EB}}>0,\\[1.2em]
  1, & B_{\mathrm{EB}}=0,\ C_{\mathrm{EB}}<0,\\
  0, & B_{\mathrm{EB}}=0,\ C_{\mathrm{EB}}\ge 0.
  \end{cases}
\]
Moreover,
\[
  \xmse(\thetamix(\omega^\star))
  \le
  \min\{0,\xmse(\thetaeb)\}.
\]
\end{theorem}

\begin{proof}
By Theorem~\ref{thm:fixed-weight}, the oracle problem is the scalar convex
minimization
\[
  \min_{\omega\in[0,1]}
  q(\omega),
  \qquad
  q(\omega)=B_{\mathrm{EB}}\omega^2+C_{\mathrm{EB}}\omega.
\]
If $B_{\mathrm{EB}}>0$, the unconstrained minimizer is
$-C_{\mathrm{EB}}/(2B_{\mathrm{EB}})$ and projection gives the constrained
minimizer.  If $B_{\mathrm{EB}}=0$, then $q$ is linear; the minimizer is
$1$ when $C_{\mathrm{EB}}<0$ and $0$ otherwise.  Since the feasible set
contains $\omega=0$ and $\omega=1$,
\[
  q(\omega^\star)\le q(0)=0,\qquad
  q(\omega^\star)\le q(1)=\xmse(\thetaeb),
\]
which proves the dominance claim.
\end{proof}

\begin{corollary}[Oracle improvement regimes]
\label{cor:oracle-regimes}
Assume $B_{\mathrm{EB}}>0$ and write
$C_{\mathrm{EB}}=V_{\mathrm{EB}}+H_{\mathrm{EB}}$.  The minimized oracle XMSE
criterion is
\[
  q(\omega^\star)=
  \begin{cases}
  0, & C_{\mathrm{EB}}\ge 0,\\[0.4em]
  -\dfrac{C_{\mathrm{EB}}^2}{4B_{\mathrm{EB}}},
    & -2B_{\mathrm{EB}}\le C_{\mathrm{EB}}<0,\\[1em]
  B_{\mathrm{EB}}+C_{\mathrm{EB}}, & C_{\mathrm{EB}}<-2B_{\mathrm{EB}}.
  \end{cases}
\]
Consequently the oracle rule returns ML when the linear XMSE component is
nonnegative, uses a strict mixture when the component is negative but not large
enough to offset the full EB bias penalty, and returns EB when the linear
benefit dominates that penalty.
\end{corollary}

\begin{proof}
For $B_{\mathrm{EB}}>0$, the unconstrained minimizer of
$q(\omega)=B_{\mathrm{EB}}\omega^2+C_{\mathrm{EB}}\omega$ is
$-C_{\mathrm{EB}}/(2B_{\mathrm{EB}})$.  If $C_{\mathrm{EB}}\ge0$, its
projection is $0$ and $q(\omega^\star)=0$.  If
$-2B_{\mathrm{EB}}\le C_{\mathrm{EB}}<0$, the unconstrained minimizer lies in
$[0,1]$, and substitution gives
$q(\omega^\star)=-C_{\mathrm{EB}}^2/(4B_{\mathrm{EB}})$.  If
$C_{\mathrm{EB}}<-2B_{\mathrm{EB}}$, the projection is $1$ and
$q(\omega^\star)=B_{\mathrm{EB}}+C_{\mathrm{EB}}$.
\end{proof}

\section{Scaled Kernel Special Case}

For implementation and experiments we consider the scaled-kernel family
\[
  P(\eta)=\eta K,\qquad K\succ0,\quad \eta>0,
\]
where $K$ is fixed.  The fixed kernels used in this paper are, for
$k,l=1,\ldots,n$ and a decay parameter $\gamma\in(0,1)$,
\[
\begin{aligned}
  K_{\mathrm{RI}}[k,l] &= 1\{k=l\},\\
  K_{\mathrm{DI}}[k,l] &= \gamma^{\,k-1}\,1\{k=l\},\\
  K_{\mathrm{TC}}[k,l] &= \gamma^{\max(k,l)},\\
  K_{\mathrm{SS}}[k,l] &= \frac{\gamma^{\,k+l+\max(k,l)}}{2}
    -\frac{\gamma^{\,3\max(k,l)}}{6},
\end{aligned}
\]
i.e., ridge (RI), diagonal-decay (DI), tuned/correlated (TC), and stable
spline (SS), matching the kernel definitions in
\cite{chen2012,pillonetto2010,ju2026xmse}.  In particular, $K_{\mathrm{SS}}$
divides the whole first term by two, not the exponent; at $n=20$ and
$\gamma=0.95$ its condition number is $7.55\times10^6$, which reproduces the
ill-conditioning reported for the SS kernel in \cite{ju2026xmse}.
Let $K^{-1}$ be denoted by $Q$ and define
\[
  S_1 = \Sigma^{-1},\qquad S_2=\Sigma^{-2}.
\]
For the scaled EB hyperparameter estimator, the limiting scale is
\[
  \eta_b^\star
  =
  \frac{\theta_0^\top Q\theta_0}{\alpha n},
\]
where $\alpha>0$ is a fixed normalization constant of the scale rule, set to
$\alpha=1$ in all experiments.  For the SURE/GCV-type hyperparameter
estimator, the limiting scale is
\[
  \eta_y^\star
  =
  \frac{\theta_0^\top Q S_1 Q\theta_0}
       {\alpha \Tr(S_1Q)}.
\]
For either choice of $\eta^\star$, the fixed-hyperparameter bias and variance
components of the regularized estimator are
\[
  B =
  \frac{\sigma^4}{(\eta^\star)^2}
  \theta_0^\top Q S_2 Q\theta_0,
\]
and
\[
  V =
  -\frac{2\sigma^4}{\eta^\star}
  \Tr(S_1QS_1).
\]
The hyperparameter-estimation component depends on the hyperparameter
estimator.  For scaled EB,
\[
  H_b =
  \frac{4\sigma^4}
       {\alpha n(\eta_b^\star)^2}
  \theta_0^\top Q S_2 Q\theta_0.
\]
For SURE/GCV-type scale estimation,
\[
  H_y =
  \frac{4\sigma^4}
       {\alpha\Tr(S_1Q)(\eta_y^\star)^2}
  \theta_0^\top Q S_1 Q S_2 Q\theta_0.
\]
The last quadratic form need not be nonnegative when $Q$ and $\Sigma$ do not
commute, which matches the phenomenon observed in the reference XMSE analysis.
In the notation of Assumption~\ref{ass:inherited}, these $B$, $V$, and $H$ are
the scaled-kernel instances of the second-order components $B_{\mathrm{EB}}$,
$V_{\mathrm{EB}}$, and $H_{\mathrm{EB}}$; the plug-in consistency statements
below ($\widehat{B}\xrightarrow{p}B_{\mathrm{EB}}$, etc.) are to be read under
this identification.

\begin{remark}[Implementation of the hyperparameter estimators]
\label{rem:one-step}
The experiments implement the scaled EB and SURE-type hyperparameter
estimates as the explicit one-step plug-in rules
\[
  \widehat{\eta}_b
  =
  \frac{\thetaml^\top Q\thetaml}{\alpha n},
  \qquad
  \widehat{\eta}_y
  =
  \frac{\thetaml^\top Q\widehat{\Sigma}^{-1}Q\thetaml}
       {\alpha\Tr(\widehat{\Sigma}^{-1}Q)},
\]
with $\widehat{\Sigma}=\Phi^\top\Phi/N$, i.e., the limiting scale maps
$\eta_b^\star$ and $\eta_y^\star$ evaluated at $\thetaml$.  These are
\emph{not} the finite-sample minimizers of the EB and SURE criteria of
\cite{ju2026xmse}; the two implementations differ at finite $N$.  They
nevertheless have the same XMSE: by the reference analysis, the XMSE of
$\widehat{\theta}_{\mathrm{EB}}(\widehat{\eta}(\thetaml))$ depends on the
hyperparameter map only through its probability limit $\eta^\star(\theta_0)$
and the derivative of the estimator map at $\theta_0$, and for the
scaled-kernel family the criterion minimizers and the one-step rules above
share both quantities (cf.\ Corollaries~2.3--2.6 of \cite{ju2026xmse}, where
the SURE and GCV minimizers are also shown to share the same limit and
XMSE).  Finite-sample GCV, by contrast, is implemented as an actual criterion
minimization, by golden-section search on $\log\eta$, so that the experiments
contain at least one criterion-minimizing implementation; its limiting scale
is again $\eta_y^\star$.  This asymmetry---closed-form one-step rules for
EB/SURE, numerical criterion minimization for GCV---is deliberate: it keeps
the EB/SURE rows exactly reproducible in closed form while letting the GCV
rows probe the finite-sample behavior of a minimized criterion.  The plug-in
XMSE components replace $(\theta_0,\Sigma,\eta^\star)$ by
$(\thetaml,\widehat{\Sigma},\widehat{\eta})$ in the displayed formulas.
\end{remark}

\begin{theorem}[Primitive plug-in consistency for scaled kernels]
\label{thm:primitive-consistency}
Fix $n$ and a positive definite kernel $K$ with $Q=K^{-1}$.  Assume
$\widehat{\Sigma}=\Phi^\top\Phi/N\to\Sigma\succ0$ and
$\thetaml\xrightarrow{p}\theta_0$.  Assume also that $\sigma^2$ is known, or is
replaced by an estimator $\widehat{\sigma}^2\xrightarrow{p}\sigma^2$.

For the scaled EB rule, suppose $\theta_0^\top Q\theta_0>0$ and define the
plug-in scale by
\[
  \widehat{\eta}_b
  =
  \frac{\thetaml^\top Q\thetaml}{\alpha n}.
\]
For the SURE/GCV-type rule, suppose
$\theta_0^\top Q\Sigma^{-1}Q\theta_0>0$ and define
\[
  \widehat{\eta}_y
  =
  \frac{\thetaml^\top Q\widehat{\Sigma}^{-1}Q\thetaml}
       {\alpha\Tr(\widehat{\Sigma}^{-1}Q)}.
\]
Then the plug-in quantities obtained from the scaled-kernel formulas satisfy
\[
  \widehat{B}\xrightarrow{p}B,\qquad
  \widehat{V}\xrightarrow{p}V,\qquad
  \widehat{H}\xrightarrow{p}H
\]
for the corresponding hyperparameter rule.  The same conclusion holds
uniformly over any finite candidate set of positive definite kernels whose
limiting scale denominators are nonzero.
\end{theorem}

\begin{proof}
The ML consistency follows from the fixed-$n$ deterministic-design regression
assumptions because
$\thetaml-\theta_0=(\Phi^\top\Phi)^{-1}\Phi^\top E=O_p(N^{-1/2})$.
Matrix inversion is continuous on the positive definite cone, so
$\widehat{\Sigma}^{-1}\to\Sigma^{-1}$ and
$\widehat{\Sigma}^{-2}\to\Sigma^{-2}$.  The quadratic forms, traces, and scale
maps displayed above are continuous functions of
$(\thetaml,\widehat{\Sigma},\widehat{\sigma}^2)$ on the event where the scale
denominators are bounded away from zero.  The nonzero limiting denominator
assumptions make this event have probability tending to one.  The continuous
mapping theorem therefore gives consistency of the plug-in scale and of each
displayed plug-in component $B$, $V$, and $H$.  For a finite candidate set,
pointwise convergence for each kernel implies convergence of the maximum error
over candidates.
\end{proof}

\begin{corollary}[Uniform consistency for compact scaled-kernel families]
\label{cor:compact-scaled-kernel-uniform}
Let $\Xi$ be compact and let $K(\xi)\succ0$ be continuous on $\Xi$, with
$Q(\xi)=K(\xi)^{-1}$.  Assume the same fixed-order design conditions as above,
with $\widehat{\Sigma}\to\Sigma\succ0$ and
$\thetaml\xrightarrow{p}\theta_0$.
For the scaled EB rule assume
\[
  \inf_{\xi\in\Xi}\theta_0^\top Q(\xi)\theta_0>0,
\]
and for the SURE/GCV-type rule assume
\[
  \inf_{\xi\in\Xi}
  \theta_0^\top Q(\xi)\Sigma^{-1}Q(\xi)\theta_0>0.
\]
Then the scaled-kernel plug-in scales and the corresponding XMSE components
converge uniformly over $\Xi$.  In particular, any compact hyperparameter
search whose empirical criterion converges uniformly satisfies the component
convergence condition in
Theorem~\ref{thm:compact-multiparameter-consistency} below.
\end{corollary}

\begin{proof}
Continuity and compactness imply that the eigenvalues of $K(\xi)$ are bounded
away from zero and infinity on $\Xi$, so $Q(\xi)$ is continuous and uniformly
bounded.  The displayed nondegeneracy assumptions, together with positivity of
$\Tr(\Sigma^{-1}Q(\xi))$, keep all scale denominators bounded away from zero.
Hence the scale maps, traces, and quadratic forms defining $B(\xi)$, $V(\xi)$,
and $H(\xi)$ are uniformly continuous on a compact neighborhood of
$(\theta_0,\Sigma,\sigma^2,\xi)$.  Since
$(\thetaml,\widehat{\Sigma},\widehat{\sigma}^2)$ converges in probability to
$(\theta_0,\Sigma,\sigma^2)$, the uniform continuous mapping argument gives the
claimed sup-norm convergence.
\end{proof}

\begin{theorem}[Compact multi-parameter plug-in consistency]
\label{thm:compact-multiparameter-consistency}
Let $\Xi$ be compact.  Let $P(\xi)$ be a continuous positive definite kernel
family on $\Xi$.  Suppose the population criterion
$m(\xi;\theta_0,\Sigma,\sigma^2)$ and the component maps
$B(\xi)$, $V(\xi)$, and $H(\xi)$ are continuous on $\Xi$, and that
$m$ has a unique minimizer $\xi^\star$ in a neighborhood where the required
scale denominators are nonzero.  Let $\widehat{m}_N(\xi)$ be the same criterion
with $(\theta_0,\Sigma,\sigma^2)$ replaced by consistent estimates, and let
$\widehat{\xi}_N$ be an $o_p(1)$ approximate minimizer of $\widehat{m}_N$ over
$\Xi$.  If
\[
  \sup_{\xi\in\Xi}
  |\widehat{m}_N(\xi)-m(\xi;\theta_0,\Sigma,\sigma^2)|
  \xrightarrow{p}0,
\]
then $\widehat{\xi}_N\xrightarrow{p}\xi^\star$.  If, moreover, the plug-in
component maps converge uniformly,
\[
  \begin{aligned}
  &\sup_{\xi\in\Xi}|\widehat B_N(\xi)-B(\xi)|
  +\sup_{\xi\in\Xi}|\widehat V_N(\xi)-V(\xi)| \\
  &\qquad
  +\sup_{\xi\in\Xi}|\widehat H_N(\xi)-H(\xi)|
  \xrightarrow{p}0,
  \end{aligned}
\]
then
\[
  \begin{gathered}
  \widehat B_N(\widehat\xi_N)\xrightarrow{p}B(\xi^\star),\qquad
  \widehat V_N(\widehat\xi_N)\xrightarrow{p}V(\xi^\star),\\
  \widehat H_N(\widehat\xi_N)\xrightarrow{p}H(\xi^\star).
  \end{gathered}
\]
\end{theorem}

\begin{proof}
Uniform convergence of $\widehat m_N$ and $o_p(1)$ empirical optimality give,
for any fixed $\varepsilon>0$, that the population criterion at
$\widehat\xi_N$ is within $o_p(1)$ of its minimum.  Since $\Xi$ is compact,
$m$ is continuous, and the minimizer is unique, the separation
\[
  \inf_{\xi\in\Xi:\|\xi-\xi^\star\|\ge\varepsilon}
  \{m(\xi;\theta_0,\Sigma,\sigma^2)-m(\xi^\star;\theta_0,\Sigma,\sigma^2)\}
  >0
\]
holds for every $\varepsilon>0$.  Hence
$\widehat\xi_N\xrightarrow{p}\xi^\star$.  The uniform component convergence and
continuity of the population component maps then imply
\[
  \begin{aligned}
  |\widehat B_N(\widehat\xi_N)-B(\xi^\star)|
  &\le
  \sup_{\xi\in\Xi}|\widehat B_N(\xi)-B(\xi)|\\
  &\quad + |B(\widehat\xi_N)-B(\xi^\star)|,
  \end{aligned}
\]
and the same argument applies to $V$ and $H$.
\end{proof}

\section{Plug-in Algorithm}

The oracle weight depends on unknown quantities.  The implementable estimator
uses finite-sample analogues of the XMSE components.

\begin{enumerate}
  \item Compute $\thetaml=(\Phi^\top\Phi)^{-1}\Phi^\top Y$ and
  $\widehat{\Sigma}=\Phi^\top\Phi/N$.
  \item Select a base hyperparameter estimate $\widehat{\eta}$ using scaled
  EB, SURE, or finite-sample GCV.
  \item Form the base regularized estimate
  \[
    \thetaeb =
    [\Phi^\top\Phi+\sigma^2P(\widehat{\eta})^{-1}]^{-1}\Phi^\top Y.
  \]
  \item Compute plug-in XMSE components
  $\widehat{B}$, $\widehat{V}$, and $\widehat{H}$ by replacing
  $(\theta_0,\Sigma,\eta^\star)$ with
  $(\thetaml,\widehat{\Sigma},\widehat{\eta})$ in the scaled-kernel formulas.
  \item Set
  \[
    \widehat{\omega}
    =
    \Pi_{[0,1]}
    \left(
      -\frac{\widehat{V}+\widehat{H}}{2\widehat{B}+\rho_N}
    \right),
  \]
  with a small stabilizing $\rho_N>0$.
  \item Return
  \[
    \widehat{\theta}_{\mathrm{mix}}
    =
    \thetaml+\widehat{\omega}(\thetaeb-\thetaml).
  \]
\end{enumerate}

If several kernels are available, the same computation can be repeated for a
finite candidate set $\mathcal{K}=\{K_1,\ldots,K_M\}$.  For each candidate,
compute the plug-in minimized criterion
\[
  \widehat{q}_j
  =
  \widehat{B}_j\widehat{\omega}_j^2+
  (\widehat{V}_j+\widehat{H}_j)\widehat{\omega}_j,
\]
and select the kernel with the smallest $\widehat{q}_j$.  This is an
XMSE-based kernel-selection rule.

Near the zero-bias boundary, a thresholded version can replace the projected
ratio in Step 5: if $\widehat{B}\le\tau_N$, set
$\widehat{\omega}=1\{\widehat{V}+\widehat{H}<0\}$; otherwise use the projected
ratio.  Theorem~\ref{thm:threshold} below gives the corresponding
boundary justification.

\begin{remark}[{Why the weight is restricted to $[0,1]$}]
\label{rem:projection}
The unconstrained minimizer $-C_{\mathrm{EB}}/(2B_{\mathrm{EB}})$ of the
quadratic criterion can leave $[0,1]$: when $C_{\mathrm{EB}}>0$ it is
negative, so an unconstrained rule would extrapolate \emph{away} from EB and
attain the strictly negative criterion value
$-C_{\mathrm{EB}}^2/(4B_{\mathrm{EB}})$, whereas the projected rule returns
ML with criterion value $0$; symmetrically, when
$C_{\mathrm{EB}}<-2B_{\mathrm{EB}}$ the unconstrained minimizer exceeds $1$
and projection forfeits $(C_{\mathrm{EB}}+2B_{\mathrm{EB}})^2/
(4B_{\mathrm{EB}})$ relative to over-shrinking beyond EB.  We nevertheless
restrict the weight to convex combinations, for three reasons.  First, the
quadratic criterion is a second-order asymptotic surrogate whose accuracy we
can only argue for on the segment between the two anchor estimators; an
extrapolated weight amplifies exactly the higher-order terms that the
expansion drops.  Second, on $[0,1]$ the rule degrades gracefully under
component-estimation error, since
$\|\thetamix(\omega)-\thetaml\|_2\le\|\thetaeb-\thetaml\|_2$ and the convex
criterion satisfies $q(\omega)\le\max\{0,\xmse(\thetaeb)\}$ for all
$\omega\in[0,1]$, while extrapolation has no such bound and turns weight
estimation error into unbounded risk inflation.  Third, the projected rule
preserves the safeguard semantics of the method: $\widehat{\omega}$ answers
``how much of the EB correction should be accepted,'' and a harmful
correction is rejected rather than reversed.  The experiments report how
often the unconstrained plug-in ratio actually leaves $[0,1]$, so the cost of
the projection is visible rather than hidden.
\end{remark}

\begin{remark}[Numerical choices in the experiments]
\label{rem:numerical-choices}
The theory lets $\rho_N\downarrow0$ at any rate; the experiments use the
fixed value $\rho=10^{-10}$, which acts purely as a division guard and is
many orders of magnitude below every observed $2\widehat{B}$, so replacing it
by a decreasing sequence would not change any reported digit.  The
thresholded boundary rule with $\tau_N$ exists for the zero-bias boundary
case $B_{\mathrm{EB}}=0$; in all reported experiments
$\widehat{B}>0$ held in every realization, so the threshold branch was never
active, and the reported results use the projected-ratio rule throughout.
\end{remark}

\begin{remark}[Finite-sample bias of the plug-in components and a
trace-corrected variant]
\label{rem:corrected}
The plug-in components evaluate quadratic forms at $\thetaml$.  Since
$\thetaml\sim\mathcal{N}(\theta_0,W)$ with $W=\sigma^2(\Phi^\top\Phi)^{-1}$,
\[
  \E[\thetaml^\top A\thetaml]
  =
  \theta_0^\top A\theta_0+\Tr(AW),
\]
so each plug-in quadratic form carries a systematic $O(\sigma^2/N)$ trace
term---upward for positive semidefinite $A$---that propagates into
$\widehat{\eta}$, $\widehat{B}$, and $\widehat{H}$ and is one identifiable
source of finite-sample miscalibration of the selected weight.  The
\emph{trace-corrected} variant therefore replaces every plug-in quadratic
form, including the one inside the plug-in scale $\widehat{\eta}$, by its
unbiased version $\thetaml^\top A\thetaml-\Tr(A\widehat{W})$ with
$\widehat{W}=\widehat{\sigma}^2(\Phi^\top\Phi)^{-1}$, floored at zero
whenever the population quantity is nonnegative.  The corrected
$\widehat{B}$ can then reach zero, in which case the projected ratio with
$\rho_N$ degenerates to the $0/1$ decision of the thresholded rule, which is
the intended boundary behavior.  The base estimator $\thetaeb$ is left
unchanged, so the corrected variant isolates the calibration of the mixing
weight.  This correction follows the same finite-sample motivation as the
approximate-XMSE refinement in Section~IV of \cite{ju2026xmse}; the
experiments below report it as a separate column, and the public benchmarks
use it to test whether the identified SURE-row failure mode is a
weight-calibration artifact.
\end{remark}

\begin{theorem}[Plug-in consistency away from the zero-bias boundary]
\label{thm:plugin-consistency}
Let $\widehat{B}$, $\widehat{V}$, and $\widehat{H}$ be finite-sample plug-in
estimators of $B_{\mathrm{EB}}$, $V_{\mathrm{EB}}$, and $H_{\mathrm{EB}}$,
respectively, and assume
\[
  \widehat{B}\xrightarrow{p}B_{\mathrm{EB}},\qquad
  \widehat{V}\xrightarrow{p}V_{\mathrm{EB}},\qquad
  \widehat{H}\xrightarrow{p}H_{\mathrm{EB}}.
\]
Let $\rho_N>0$ be deterministic with $\rho_N\downarrow0$ and define
\[
  \widehat{\omega}_{\mathrm{XMSE}}
  =
  \Pi_{[0,1]}
  \left(
    -\frac{\widehat{V}+\widehat{H}}{2\widehat{B}+\rho_N}
  \right).
\]
If $B_{\mathrm{EB}}>0$, then
\[
  \widehat{\omega}_{\mathrm{XMSE}}\xrightarrow{p}\omega^\star.
\]
\end{theorem}

\begin{proof}
Let $C_{\mathrm{EB}}=V_{\mathrm{EB}}+H_{\mathrm{EB}}$ and
$\widehat{C}=\widehat{V}+\widehat{H}$.  By assumption,
$\widehat{C}\xrightarrow{p}C_{\mathrm{EB}}$.  Since $B_{\mathrm{EB}}>0$ and
$\rho_N\to0$,
\[
  2\widehat{B}+\rho_N \xrightarrow{p} 2B_{\mathrm{EB}}>0.
\]
Therefore
\[
  -\frac{\widehat{C}}{2\widehat{B}+\rho_N}
  \xrightarrow{p}
  -\frac{C_{\mathrm{EB}}}{2B_{\mathrm{EB}}}.
\]
The projection $\Pi_{[0,1]}$ is continuous, so the continuous mapping theorem
gives the claim.
\end{proof}

\begin{theorem}[Plug-in oracle inequality]
\label{thm:oracle-inequality}
Let
\[
  q(\omega)=B_{\mathrm{EB}}\omega^2+
  (V_{\mathrm{EB}}+H_{\mathrm{EB}})\omega
\]
and
\[
  \widehat{q}(\omega)=\widehat{B}\omega^2+
  (\widehat{V}+\widehat{H})\omega.
\]
Let $\omega^\star\in\arg\min_{\omega\in[0,1]}q(\omega)$ and let
$\widehat{\omega}\in\arg\min_{\omega\in[0,1]}\widehat{q}(\omega)$, with an
arbitrary deterministic tie-breaking rule.  Define
\[
  e_N=\sup_{\omega\in[0,1]}|\widehat{q}(\omega)-q(\omega)|.
\]
Then, deterministically,
\[
  q(\widehat{\omega})-q(\omega^\star)\le 2e_N.
\]
Consequently, if
$|\widehat{B}-B_{\mathrm{EB}}|+
|\widehat{V}-V_{\mathrm{EB}}|+
|\widehat{H}-H_{\mathrm{EB}}|=O_p(r_N)$ for some deterministic
$r_N\downarrow0$, then
\[
  q(\widehat{\omega})-\min_{\omega\in[0,1]}q(\omega)=O_p(r_N).
\]
\end{theorem}

\begin{proof}
The optimality of $\widehat{\omega}$ for $\widehat{q}$ gives
$\widehat{q}(\widehat{\omega})\le\widehat{q}(\omega^\star)$.  Therefore
\[
\begin{aligned}
q(\widehat{\omega})-q(\omega^\star)
&=
\{q(\widehat{\omega})-\widehat{q}(\widehat{\omega})\}
+\{\widehat{q}(\widehat{\omega})-\widehat{q}(\omega^\star)\} \\
&\quad
+\{\widehat{q}(\omega^\star)-q(\omega^\star)\} \\
&\le
|q(\widehat{\omega})-\widehat{q}(\widehat{\omega})|
+|\widehat{q}(\omega^\star)-q(\omega^\star)|
\le 2e_N.
\end{aligned}
\]
For $\omega\in[0,1]$,
\[
  |\widehat{q}(\omega)-q(\omega)|
  \le
  |\widehat{B}-B_{\mathrm{EB}}|
  +|\widehat{V}-V_{\mathrm{EB}}|
  +|\widehat{H}-H_{\mathrm{EB}}|,
\]
so the stochastic statement follows.
\end{proof}

\begin{theorem}[Plug-in oracle regret rate]
\label{thm:regret-rate}
Under the assumptions of Theorem~\ref{thm:plugin-consistency}, define
\[
  q(\omega)=B_{\mathrm{EB}}\omega^2+
  (V_{\mathrm{EB}}+H_{\mathrm{EB}})\omega.
\]
Assume in addition that $B_{\mathrm{EB}}>0$,
$\omega^\star\in(0,1)$, and for some deterministic sequence $r_N\downarrow0$,
\[
  |\widehat{B}-B_{\mathrm{EB}}|
  +
  |\widehat{V}-V_{\mathrm{EB}}|
  +
  |\widehat{H}-H_{\mathrm{EB}}|
  +
  \rho_N
  =
  O_p(r_N).
\]
Then
\[
  \widehat{\omega}_{\mathrm{XMSE}}-\omega^\star=O_p(r_N)
\]
and
\[
  q(\widehat{\omega}_{\mathrm{XMSE}})
  -
  q(\omega^\star)
  =
  O_p(r_N^2).
\]
\end{theorem}

\begin{proof}
Let $C_{\mathrm{EB}}=V_{\mathrm{EB}}+H_{\mathrm{EB}}$ and
$\widehat{C}=\widehat{V}+\widehat{H}$.  Since
$\omega^\star\in(0,1)$, the projection is inactive with probability tending to
one.  On this event,
\[
\begin{aligned}
\widehat{\omega}_{\mathrm{XMSE}}-\omega^\star
&=
-\frac{\widehat{C}}{2\widehat{B}+\rho_N}
+\frac{C_{\mathrm{EB}}}{2B_{\mathrm{EB}}}  \\
&=
\frac{C_{\mathrm{EB}}(2\widehat{B}+\rho_N)
-2B_{\mathrm{EB}}\widehat{C}}
{2B_{\mathrm{EB}}(2\widehat{B}+\rho_N)}.
\end{aligned}
\]
The denominator is bounded away from zero with probability tending to one, and
the numerator is $O_p(r_N)$.  Hence
$\widehat{\omega}_{\mathrm{XMSE}}-\omega^\star=O_p(r_N)$.
Because $\omega^\star$ is the unconstrained minimizer in the interior,
$q'(\omega^\star)=0$ and
\[
  q(\widehat{\omega}_{\mathrm{XMSE}})-q(\omega^\star)
  =
  B_{\mathrm{EB}}
  (\widehat{\omega}_{\mathrm{XMSE}}-\omega^\star)^2
  =
  O_p(r_N^2).
\]
\end{proof}

For each \emph{fixed} weight $\omega\in[0,1]$, define the deterministic
fixed-weight risk curve
\[
  \Delta_N(\omega)
  =
  N^2\{
  \mathrm{MSE}(\thetamix(\omega))
  -
  \mathrm{MSE}(\thetaml)
  \},
\]
where the expectations are taken over all randomness with the weight held
fixed at $\omega$.  Thus $\omega\mapsto\Delta_N(\omega)$ is a deterministic
function, and Theorem~\ref{thm:fixed-weight} states its pointwise convergence
to $q(\omega)$.  The next result upgrades this to uniform convergence on
$[0,1]$, using no assumption beyond Assumption~\ref{ass:inherited}.

\begin{theorem}[Uniform fixed-weight convergence]
\label{thm:uniform-curve}
Under Assumption~\ref{ass:inherited},
\[
  \varepsilon_N
  :=
  \sup_{\omega\in[0,1]}
  |\Delta_N(\omega)-q(\omega)|
  \longrightarrow 0.
\]
\end{theorem}

\begin{proof}
For every $N$ the curve is exactly a quadratic polynomial in $\omega$,
\[
  \Delta_N(\omega)=a_N\,\omega+b_N\,\omega^2,
\]
with coefficients
$a_N=2N^2\,\E\{(\thetaml-\theta_0)^\top d_N\}$ and
$b_N=N^2\,\E\,\|d_N\|_2^2$.
Assumption~\ref{ass:inherited} gives
$a_N\to V_{\mathrm{EB}}+H_{\mathrm{EB}}=C_{\mathrm{EB}}$ and
$b_N=N^2\|\E[d_N]\|_2^2+N^2\E\|d_N-\E[d_N]\|_2^2\to B_{\mathrm{EB}}$.
Hence, for all $\omega\in[0,1]$,
\[
  |\Delta_N(\omega)-q(\omega)|
  \le
  |a_N-C_{\mathrm{EB}}|+|b_N-B_{\mathrm{EB}}|,
\]
and the right-hand side is a deterministic null sequence.
\end{proof}

\begin{theorem}[Oracle transfer along the fixed-weight risk curve]
\label{thm:plugin-xmse-transfer}
Let the assumptions of Theorem~\ref{thm:regret-rate} hold, and let
$\varepsilon_N$ be as in Theorem~\ref{thm:uniform-curve}.  Then the
fixed-weight risk curve evaluated at the selected weight satisfies
\[
  \Delta_N(\widehat{\omega}_{\mathrm{XMSE}})
  -
  q(\omega^\star)
  =
  O_p(r_N^2)+O(\varepsilon_N)
\]
and
\[
  \Delta_N(\widehat{\omega}_{\mathrm{XMSE}})
  -
  \Delta_N(\omega^\star)
  =
  O_p(r_N^2)+O(\varepsilon_N).
\]
\end{theorem}

\begin{proof}
Since $\widehat{\omega}_{\mathrm{XMSE}}\in[0,1]$ always,
$|\Delta_N(\widehat{\omega}_{\mathrm{XMSE}})
-q(\widehat{\omega}_{\mathrm{XMSE}})|\le\varepsilon_N$ surely, by
Theorem~\ref{thm:uniform-curve}.  Theorem~\ref{thm:regret-rate} gives
$q(\widehat{\omega}_{\mathrm{XMSE}})-q(\omega^\star)=O_p(r_N^2)$, and adding
the two bounds proves the first claim.  The second claim follows by also
applying $|\Delta_N(\omega^\star)-q(\omega^\star)|\le\varepsilon_N$.
\end{proof}

\begin{remark}[What the transfer result does and does not control]
\label{rem:transfer-scope}
$\Delta_N(\widehat{\omega}_{\mathrm{XMSE}})$ is the deterministic risk curve
evaluated at the realized weight: it is the second-order risk the mixed
estimator would have if the selected weight were frozen and reused on
independent data.  It is \emph{not} the XMSE of the data-dependent estimator
$\thetamix(\widehat{\omega}_{\mathrm{XMSE}})$, in which the random weight
sits inside the expectation.  The difference consists of cross-moments
between $\widehat{\omega}_{\mathrm{XMSE}}$ and
$(\thetaml-\theta_0,\thetaeb-\thetaml)$, which arise because the weight and
the EB correction are computed from the same data.  Controlling these
cross-moments requires a joint expansion of the weight map and the EB
correction; for the scaled-kernel one-step rules of
Remark~\ref{rem:one-step}, $\widehat{\omega}_{\mathrm{XMSE}}$ is a smooth
function of $\thetaml$ away from the projection boundary, so a delta-method
analysis appears feasible, but it is beyond the scope of this paper and is
recorded as an open problem in the Limitations section.
\end{remark}

\begin{theorem}[Finite-candidate kernel selection]
\label{thm:finite-kernel-selection}
Consider a finite set of candidate kernels
$\mathcal{K}=\{K_1,\ldots,K_M\}$.  Let $q_j^\star$ denote the oracle minimized
XMSE criterion for candidate $K_j$, and let $\widehat{q}_j$ be its plug-in
counterpart.  Assume
\[
  \widehat{q}_j\xrightarrow{p}q_j^\star,\qquad j=1,\ldots,M.
\]
If the oracle minimizer is unique, i.e., there exists $j^\star$ such that
\[
  q_{j^\star}^\star < \min_{j\ne j^\star} q_j^\star,
\]
then the plug-in selector
\[
  \widehat{j}=\arg\min_{1\le j\le M}\widehat{q}_j
\]
satisfies
\[
  \Pr(\widehat{j}=j^\star)\to1.
\]
\end{theorem}

\begin{proof}
Let
\[
  g=\min_{j\ne j^\star}(q_j^\star-q_{j^\star}^\star)>0.
\]
Since $M$ is finite and $\widehat{q}_j\xrightarrow{p}q_j^\star$ for each
$j$, we have
\[
  \max_{1\le j\le M}|\widehat{q}_j-q_j^\star|\xrightarrow{p}0.
\]
On the event where this maximum error is smaller than $g/3$,
\[
  \widehat{q}_{j^\star}
  <
  q_{j^\star}^\star+g/3
  <
  q_j^\star-g/3
  <
  \widehat{q}_j
\]
for every $j\ne j^\star$.  Thus $\widehat{j}=j^\star$ on an event whose
probability tends to one.
\end{proof}

\begin{corollary}[Finite-set oracle inequality]
\label{cor:finite-set-oracle-inequality}
In the setting of Theorem~\ref{thm:finite-kernel-selection}, let
\[
  j^\star\in\arg\min_{1\le j\le M} q_j^\star
\]
be any oracle minimizer and define
\[
  e_N=\max_{1\le j\le M}|\widehat{q}_j-q_j^\star|.
\]
Then the selected candidate satisfies the deterministic inequality
\[
  q_{\widehat{j}}^\star
  -
  q_{j^\star}^\star
  \le 2e_N.
\]
Consequently, if $e_N=O_p(r_N)$ for some deterministic $r_N\downarrow0$, then
\[
  q_{\widehat{j}}^\star
  -
  \min_{1\le j\le M}q_j^\star
  =
  O_p(r_N).
\]
\end{corollary}

\begin{proof}
By definition of $\widehat{j}$,
\[
  \widehat{q}_{\widehat{j}}\le \widehat{q}_{j^\star}.
\]
Adding and subtracting the plug-in criteria gives
\[
\begin{aligned}
q_{\widehat{j}}^\star-q_{j^\star}^\star
&=
(q_{\widehat{j}}^\star-\widehat{q}_{\widehat{j}})
+(\widehat{q}_{\widehat{j}}-\widehat{q}_{j^\star})
+(\widehat{q}_{j^\star}-q_{j^\star}^\star)\\
&\le
|q_{\widehat{j}}^\star-\widehat{q}_{\widehat{j}}|
+|\widehat{q}_{j^\star}-q_{j^\star}^\star|
\le 2e_N.
\end{aligned}
\]
The stochastic statement follows immediately from the assumed rate for $e_N$.
\end{proof}

\begin{theorem}[Growing candidate dictionary]
\label{thm:growing-dictionary}
For each sample size, let
\[
  \mathcal{K}_N=\{K_{N,1},\ldots,K_{N,M_N}\}
\]
be a deterministic candidate dictionary.  Let $q_{N,j}^\star$ denote the
oracle minimized XMSE criterion for candidate $K_{N,j}$ and let
$\widehat{q}_{N,j}$ be its plug-in counterpart.
Define
\[
  \widehat{j}_N=\arg\min_{1\le j\le M_N}\widehat{q}_{N,j},
  \qquad
  j_N^\star\in\arg\min_{1\le j\le M_N}q_{N,j}^\star,
\]
and
\[
  e_N=\max_{1\le j\le M_N}
  |\widehat{q}_{N,j}-q_{N,j}^\star|.
\]
Then the selected candidate satisfies
\[
  q_{N,\widehat{j}_N}^\star
  -
  q_{N,j_N^\star}^\star
  \le 2e_N.
\]
If $e_N=o_p(1)$, the selected candidate is asymptotically oracle optimal over
the growing dictionary.  If, in addition, the oracle minimizer is unique with
gap
\[
  g_N
  =
  \min_{j\ne j_N^\star}
  (q_{N,j}^\star-q_{N,j_N^\star}^\star)>0
\]
and $e_N=o_p(g_N)$, then
\[
  \Pr(\widehat{j}_N=j_N^\star)\to1.
\]
\end{theorem}

\begin{proof}
The oracle inequality is the same add-and-subtract argument used for the
finite-set oracle inequality:
\[
\begin{aligned}
q_{N,\widehat{j}_N}^\star-q_{N,j_N^\star}^\star
&=
(q_{N,\widehat{j}_N}^\star-\widehat{q}_{N,\widehat{j}_N})
+(\widehat{q}_{N,\widehat{j}_N}-\widehat{q}_{N,j_N^\star})
\\
&\quad
+(\widehat{q}_{N,j_N^\star}-q_{N,j_N^\star}^\star)\\
&\le
|q_{N,\widehat{j}_N}^\star-\widehat{q}_{N,\widehat{j}_N}|
+|\widehat{q}_{N,j_N^\star}-q_{N,j_N^\star}^\star|
\le 2e_N.
\end{aligned}
\]
Thus $e_N=o_p(1)$ implies asymptotic oracle optimality.  On the event
$e_N<g_N/3$, for every $j\ne j_N^\star$,
\[
  \widehat{q}_{N,j_N^\star}
  <
  q_{N,j_N^\star}^\star+g_N/3
  \le
  q_{N,j}^\star-g_N/3
  <
  \widehat{q}_{N,j}.
\]
Hence $\widehat{j}_N=j_N^\star$ on this event, and
$e_N=o_p(g_N)$ makes its probability tend to one.
\end{proof}

\begin{corollary}[High-probability dictionary oracle bound]
\label{cor:high-prob-dictionary}
In the setting of Theorem~\ref{thm:growing-dictionary}, suppose that for
deterministic sequences $a_N\downarrow0$ and $\beta_N\downarrow0$,
\[
  \Pr\{|\widehat{q}_{N,j}-q_{N,j}^\star|>a_N\}
  \le \beta_N,\qquad j=1,\ldots,M_N.
\]
Then, with probability at least $1-M_N\beta_N$,
\[
  q_{N,\widehat{j}_N}^\star
  -
  \min_{1\le j\le M_N}q_{N,j}^\star
  \le 2a_N.
\]
Consequently, if $a_N\to0$ and $M_N\beta_N\to0$, the selected dictionary
element is asymptotically oracle optimal.  If the oracle gap $g_N$ is
positive and $a_N<g_N/3$, then the same event implies exact oracle
selection, $\widehat{j}_N=j_N^\star$.
\end{corollary}

\begin{proof}
By the union bound,
\[
  \Pr(e_N>a_N)
  \le
  \sum_{j=1}^{M_N}
  \Pr\{|\widehat{q}_{N,j}-q_{N,j}^\star|>a_N\}
  \le M_N\beta_N.
\]
On the complementary event, Theorem~\ref{thm:growing-dictionary} gives the
oracle inequality with $e_N\le a_N$.  If $a_N<g_N/3$, the gap argument in
the proof of Theorem~\ref{thm:growing-dictionary} gives
$\widehat{j}_N=j_N^\star$.
\end{proof}

\begin{corollary}[Sub-Gaussian dictionary rate]
\label{cor:subgaussian-dictionary-rate}
In the setting of Theorem~\ref{thm:growing-dictionary}, suppose that for some
scale $s_N\downarrow0$ and all $t>0$, $j=1,\ldots,M_N$,
\[
  \Pr\{|\widehat{q}_{N,j}-q_{N,j}^\star|>t\}
  \le
  2\exp\!\left(-\frac{t^2}{s_N^2}\right).
\]
Then, for any $\delta\in(0,1)$, with probability at least $1-\delta$,
\[
  q_{N,\widehat{j}_N}^\star
  -
  \min_{1\le j\le M_N}q_{N,j}^\star
  \le
  2s_N\sqrt{\log\!\left(\frac{2M_N}{\delta}\right)}.
\]
Consequently, if $s_N^2\log M_N\to0$, the selected dictionary element is
asymptotically oracle optimal.
\end{corollary}

\begin{proof}
Apply Corollary~\ref{cor:high-prob-dictionary} with
\[
  a_N=s_N\sqrt{\log\!\left(\frac{2M_N}{\delta}\right)}
  \quad\text{and}\quad
  \beta_N=\delta/M_N.
\]
The stated tail bound gives the required per-candidate probability bound, and
the oracle inequality follows.  For fixed $\delta$, the right-hand side is
$o(1)$ whenever $s_N^2\log M_N\to0$.
\end{proof}

\begin{theorem}[Thresholded consistency at the zero-bias boundary]
\label{thm:threshold}
Let $C_{\mathrm{EB}}=V_{\mathrm{EB}}+H_{\mathrm{EB}}$ and
$\widehat{C}=\widehat{V}+\widehat{H}$.  Suppose
$B_{\mathrm{EB}}=0$, $C_{\mathrm{EB}}\ne0$,
$\widehat{C}\xrightarrow{p}C_{\mathrm{EB}}$, and for some deterministic
$\tau_N\downarrow0$,
\[
  \Pr(\widehat{B}\le\tau_N)\to1.
\]
Define the thresholded plug-in weight
\[
  \widetilde{\omega}
  =
  \begin{cases}
    1\{\widehat{C}<0\}, & \widehat{B}\le\tau_N,\\[3pt]
    \Pi_{[0,1]}\!\left(-\dfrac{\widehat{C}}{2\widehat{B}+\rho_N}\right),
    & \widehat{B}>\tau_N.
  \end{cases}
\]
Let
\[
  \omega_0^\star
  =
  \begin{cases}
    1, & C_{\mathrm{EB}}<0,\\
    0, & C_{\mathrm{EB}}>0,
  \end{cases}
\]
be the minimizer of the linear oracle objective
$q_0(\omega)=C_{\mathrm{EB}}\omega$ over $[0,1]$.  Then
\[
  \widetilde{\omega}\xrightarrow{p}\omega_0^\star.
\]
\end{theorem}

\begin{proof}
Let $A_N=\{\widehat{B}\le\tau_N\}$.  By assumption,
$\Pr(A_N)\to1$.  Since $\widehat{C}\xrightarrow{p}C_{\mathrm{EB}}$ and
$C_{\mathrm{EB}}\ne0$, the sign of $\widehat{C}$ is eventually correct with
probability tending to one:
\[
  \Pr\{1\{\widehat{C}<0\}=1\{C_{\mathrm{EB}}<0\}\}\to1.
\]
On $A_N$, the thresholded rule equals $1\{\widehat{C}<0\}$.  Therefore
\[
  \Pr(\widetilde{\omega}=\omega_0^\star)
  \ge
  \Pr\!\left(
    A_N\cap
    \{1\{\widehat{C}<0\}=1\{C_{\mathrm{EB}}<0\}\}
  \right)
  \to 1,
\]
which implies convergence in probability.
\end{proof}

\begin{remark}[Scope of the boundary result]
The thresholded theorem gives consistency of the selected boundary decision.
It does not claim a full XMSE expansion for the estimator with a random,
nonsmooth projected or thresholded weight.  Such an expansion would require
tracking additional terms from the weight map near the projection and threshold
boundaries.
\end{remark}

\section{Numerical Experiments}

The experiments are organized to mirror the theory.  The diagonal and
tail-mismatch simulations are controlled fixed-order FIR settings where the
XMSE expansion and plug-in component estimates are intended to be informative.
The sample-size and SNR sweeps check the asymptotic direction suggested by the
plug-in consistency and regret results.  The finite-candidate experiment tests
the risk-oriented kernel-selection criterion behind the finite-set oracle
inequality.  The Silverbox and Cascaded Tanks studies, by contrast, are
out-of-model public-data checks: they examine whether the same diagnostic
weight can retreat from harmful shrinkage on nonlinear systems and where its
finite-sample calibration limits appear, not whether the XMSE theorems hold
outside their assumptions.

\subsection{Common protocol}
\label{sec:protocol}

All synthetic experiments use the FIR model of the problem setup with known
noise variance $\sigma^2=1$; only the two public benchmarks estimate the
noise variance, as disclosed there.  For each system, the input is drawn
i.i.d.\ standard normal and the true parameter is constructed as follows and
then rescaled so that the sample signal-to-noise ratio
$\mathrm{SNR}=\overline{(\Phi\theta_0)^2}/\sigma^2$ matches its target.
Aligned and misaligned parameters are
\[
  \theta_0^{\mathrm{ali}}[k]=\sqrt{K[k,k]}\;z_k,
  \qquad
  \theta_0^{\mathrm{mis}}[k]=z_k/\sqrt{K[k,k]},
\]
with $z\sim\mathcal{N}(0,I_n)$, so that the parameter energy profile follows
or opposes the kernel diagonal.  The tail-mismatch parameters are
\[
  \theta_0^{\mathrm{tail}}[k]
  =
  r_k\,(-1)^{k+1}\,|z_k|,
  \qquad
  r_k=0.1+\frac{1.9\,(k-1)}{n-1},
\]
an increasing ramp with alternating signs, which concentrates energy in the
late impulse-response coefficients and is therefore intentionally misaligned
with the decaying TC/SS kernels.  The hyperparameter estimators are
implemented as described in Remark~\ref{rem:one-step}, the projected ratio
uses the numerical choices of Remark~\ref{rem:numerical-choices}, and
$\alpha=1$ throughout.

Two error metrics are reported.  For the synthetic experiments, the
parameter-space fit of an estimate $\widehat{\theta}$ of $\theta_0$ is
\[
  \mathrm{FIT}_\theta(\widehat{\theta})
  =
  100\left(1-
  \frac{\|\widehat{\theta}-\theta_0\|_2}
       {\|\theta_0-\bar{\theta}_0\mathbf{1}\|_2}\right),
  \qquad
  \bar{\theta}_0=\frac{1}{n}\sum_{k=1}^n\theta_0[k],
\]
averaged over Monte Carlo repetitions.  For the public benchmarks, the
output-prediction fit on a test record $y$ with prediction $\widehat{y}$ is
\[
  \mathrm{FIT}_y(\widehat{y})
  =
  100\left(1-
  \frac{\|y-\widehat{y}\|_2}{\|y-\bar{y}\mathbf{1}\|_2}\right).
\]
The two definitions agree in form but live in different spaces; tables state
which one is used.

\subsection{Diagonal-kernel calibration}

We first run a calibration experiment with $n=20$, $N=50$, SNR 10, 50
systems, and 200 Monte Carlo repetitions per system, using the RI kernel with
neutral parameters ($\theta_0=z$) and the diagonal-decay kernel
$K_{\mathrm{DI}}$ with aligned and misaligned parameters.
Table~\ref{tab:diag} reports the results.

\begin{table*}[t]
\centering
\small
\begin{tabular}{lccccccc}
\hline
Setting & ML & base & oracle mix & plug-in mix & $\hat\omega$ & base $<$ ML & mix $<$ ML \\
\hline
RI-EB-neutral & 0.898 & 0.818 & 0.811 & 0.830 & 0.824 & 96\% & 96\% \\
DIAG\_DECAY-EB-aligned & 0.798 & 0.720 & 0.716 & 0.730 & 0.798 & 94\% & 98\% \\
DIAG\_DECAY-EB-misaligned & 0.892 & 0.826 & 0.815 & 0.840 & 0.696 & 94\% & 98\% \\
\hline
\end{tabular}

\caption{Diagonal-kernel calibration experiment, $N=50$, $n=20$, SNR 10,
known $\sigma^2$.  Mean MSE across 50 systems; the last two columns give the
fraction of systems where the base and plug-in mixed estimators have lower
MSE than ML.}
\label{tab:diag}
\end{table*}

In these diagonal settings the regularized estimator helps on average in both
alignments, and the plug-in mixed estimator retains most of the benefit while
giving up a small part of it (e.g., $0.730$ versus $0.720$ for the aligned
diagonal kernel against $0.798$ for ML).  The misaligned construction reduces
but does not eliminate the benefit of shrinkage at this sample size, so this
experiment functions as a calibration check of the plug-in weight in benign
conditions; the genuinely harmful-regularization regime is produced by the
tail-mismatch construction studied next.

\subsection{Tail mismatch}

The main experiment considers TC and SS kernels with $\gamma=0.95$ and the
tail-mismatch parameter class of Section~\ref{sec:protocol}, whose energy is
concentrated in later impulse-response coefficients and which is therefore
intentionally misaligned with decaying kernels.  We use 100 systems and 500
Monte Carlo repetitions per system; the TC and SS rows use the same seed, and
since the tail construction does not depend on the kernel, the two kernels
see identical systems and noise realizations, so cross-kernel comparisons are
paired.  Table~\ref{tab:tail} reports mean MSEs.  The column ``base'' denotes
the regularized estimator using scaled EB, SURE-type, or finite-sample GCV
scale selection.

\begin{table*}[t]
\centering
\small
\begin{tabular}{lccccccc}
\hline
Setting & ML & base & oracle mix & plug-in mix & $\hat\omega$ & base $<$ ML & mix $<$ ML \\
\hline
TC-EB & 0.928 & 1.001 & 0.902 & 0.910 & 0.437 & 43\% & 67\% \\
TC-SURE & 0.928 & 0.912 & 0.908 & 0.911 & 0.602 & 68\% & 68\% \\
TC-GCV & 0.928 & 0.914 & 0.908 & 0.912 & 0.585 & 62\% & 73\% \\
SS-EB & 0.928 & 1.071 & 0.924 & 0.928 & 0.155 & 10\% & 42\% \\
SS-SURE & 0.928 & 0.929 & 0.926 & 0.927 & 0.419 & 49\% & 42\% \\
SS-GCV & 0.928 & 0.930 & 0.926 & 0.927 & 0.423 & 43\% & 57\% \\
\hline
\end{tabular}

\caption{Mean MSE under TC/SS tail mismatch, $N=50$, $n=20$, SNR 10, known
$\sigma^2$.  The last two columns give the fraction of systems where the base
and plug-in mixed estimators have lower MSE than ML.  TC and SS rows share
the same systems and noise realizations.}
\label{tab:tail}
\end{table*}

\begin{table*}[t]
\centering
\small
\begin{tabular}{lcccccc}
\hline
Setting & ML & base & plug-in mix & SURE-$\omega$ mix & hard $\{0,1\}$ & corrected mix \\
\hline
TC-EB & 0.928 & 1.001 & 0.910 & 0.911 & 0.914 & 0.910 \\
TC-SURE & 0.928 & 0.912 & 0.911 & 0.911 & 0.912 & 0.911 \\
TC-GCV & 0.928 & 0.914 & 0.912 & 0.911 & 0.913 & 0.911 \\
SS-EB & 0.928 & 1.071 & 0.928 & 0.931 & 0.931 & 0.928 \\
SS-SURE & 0.928 & 0.929 & 0.927 & 0.929 & 0.927 & 0.927 \\
SS-GCV & 0.928 & 0.930 & 0.927 & 0.929 & 0.927 & 0.927 \\
\hline
\end{tabular}

\caption{Baseline comparison under TC/SS tail mismatch (mean MSE).  SURE-$\omega$
tunes the mixing weight by minimizing the fixed-hyperparameter SURE of
$\mathrm{MSE}(\thetamix(\omega))-\mathrm{MSE}(\thetaml)$; hard $\{0,1\}$
selects ML or EB by the sign of the plug-in XMSE; corrected mix uses the
trace-corrected components of Remark~\ref{rem:corrected}.}
\label{tab:tail-baselines}
\end{table*}

\begin{table*}[t]
\centering
\small
\begin{tabular}{lcccc}
\hline
Setting & base--ML & mix--ML & mix--base & mix $<$ base \\
\hline
TC-EB & 0.0726 (0.0216) & -0.0187 (0.0039) & -0.0913 (0.0202) & 61\% \\
TC-SURE & -0.0167 (0.0040) & -0.0175 (0.0036) & -0.0008 (0.0009) & 38\% \\
TC-GCV & -0.0142 (0.0041) & -0.0167 (0.0034) & -0.0026 (0.0012) & 53\% \\
SS-EB & 0.1425 (0.0232) & -0.0006 (0.0005) & -0.1431 (0.0231) & 90\% \\
SS-SURE & 0.0007 (0.0009) & -0.0015 (0.0005) & -0.0022 (0.0007) & 53\% \\
SS-GCV & 0.0019 (0.0011) & -0.0014 (0.0004) & -0.0033 (0.0009) & 62\% \\
\hline
\end{tabular}

\caption{Mean MSE gaps with standard errors in parentheses under TC/SS tail
mismatch.  Negative values indicate that the estimator named before the dash
has lower MSE.  The per-system winning fractions against ML appear in
Table~\ref{tab:tail}.}
\label{tab:tail-gaps}
\end{table*}

\begin{table*}[t]
\centering
\begin{tabular}{lcccc}
\hline
Setting & ML & base & oracle mix & plug-in mix \\
\hline
TC-EB & 75.95 & 75.52 & 76.34 & 76.24 \\
TC-SURE & 75.95 & 76.23 & 76.26 & 76.22 \\
TC-GCV & 75.95 & 76.20 & 76.26 & 76.21 \\
SS-EB & 75.95 & 74.68 & 76.02 & 75.96 \\
SS-SURE & 75.95 & 75.96 & 75.99 & 75.97 \\
SS-GCV & 75.95 & 75.94 & 75.99 & 75.97 \\
\hline
\end{tabular}

\caption{Average parameter-space $\mathrm{FIT}_\theta$ under TC/SS tail
mismatch, $N=50$, $n=20$, SNR 10.}
\label{tab:tail-fit}
\end{table*}

\begin{table*}[t]
\centering
\begingroup
\scriptsize
\setlength{\tabcolsep}{2pt}
\begin{tabular}{lccccc}
\hline
Setting & $\hat\omega$ & ML MSE & base MSE & plug-in MSE & plug-in FIT \\
\hline
TC-EB & 0.366 [0.111, 0.804] & 0.84 [0.71, 1.05] & 0.87 [0.76, 1.12] & 0.82 [0.71, 1.03] & 77.43 [72.82, 79.69] \\
TC-SURE & 0.658 [0.199, 0.993] & 0.84 [0.71, 1.05] & 0.82 [0.71, 1.02] & 0.82 [0.71, 1.03] & 77.46 [72.80, 79.69] \\
TC-GCV & 0.639 [0.206, 0.937] & 0.84 [0.71, 1.05] & 0.82 [0.71, 1.03] & 0.82 [0.71, 1.03] & 77.45 [72.78, 79.69] \\
SS-EB & 0.112 [0.013, 0.228] & 0.84 [0.71, 1.05] & 0.94 [0.80, 1.22] & 0.84 [0.71, 1.05] & 77.29 [72.47, 79.65] \\
SS-SURE & 0.391 [0.000, 0.758] & 0.84 [0.71, 1.05] & 0.84 [0.72, 1.05] & 0.84 [0.71, 1.05] & 77.29 [72.52, 79.65] \\
SS-GCV & 0.408 [0.057, 0.707] & 0.84 [0.71, 1.05] & 0.84 [0.72, 1.06] & 0.84 [0.71, 1.05] & 77.29 [72.51, 79.65] \\
\hline
\end{tabular}

\endgroup
\caption{Median and interquartile ranges across systems under TC/SS tail
mismatch.}
\label{tab:tail-quantiles}
\end{table*}

\begin{table*}[t]
\centering
\small
\begin{tabular}{lcccccc}
\hline
Setting & base $\Delta$MSE & oracle XMSE/$N^2$ & sign & mix $\Delta$MSE & plug-in XMSE/$N^2$ & sign \\
\hline
TC-EB & 0.0726 & 0.1278 & 95\% & -0.0187 & -0.0275 & 70\% \\
TC-SURE & -0.0167 & -0.0174 & 91\% & -0.0175 & -0.0212 & 78\% \\
TC-GCV & -0.0142 & -0.0174 & 97\% & -0.0167 & -0.0200 & 73\% \\
SS-EB & 0.1425 & 0.2294 & 96\% & -0.0006 & -0.0041 & 59\% \\
SS-SURE & 0.0007 & 0.0018 & 88\% & -0.0015 & -0.0024 & 68\% \\
SS-GCV & 0.0019 & 0.0018 & 90\% & -0.0014 & -0.0023 & 57\% \\
\hline
\end{tabular}

\caption{Sample MSE gaps and XMSE-based diagnostics under TC/SS tail
mismatch.  The base diagnostic is the \emph{oracle} XMSE evaluated at
$\theta_0$, whereas the mix diagnostic is the \emph{plug-in} XMSE evaluated
at $\thetaml$ including the data-dependent weight; the sign columns give the
fraction of systems where the diagnostic sign matches the sample MSE-gap
sign.  These diagnostics implement, at the level of our mixed rule, the
finite-sample XMSE approximation idea of Section~IV of \cite{ju2026xmse}.}
\label{tab:tail-xmse}
\end{table*}

The scaled EB rows show the clearest protection effect.  With the TC kernel
the base estimator is $8\%$ worse than ML on average ($1.001$ versus
$0.928$) while the plug-in mixed estimator is $2\%$ better ($0.910$); with
the ill-conditioned SS kernel the base estimator is $15\%$ worse ($1.071$)
and the plug-in mixed estimator retreats essentially to ML ($0.928$, mean
weight $0.155$).  SURE-type and GCV scale selection are intrinsically more
conservative in these settings, with base estimators close to or slightly
better than ML, and the mixed estimator matches or slightly improves them.
In terms of mean MSE the plug-in mixed estimator is at or below ML in all six
settings (Table~\ref{tab:tail-gaps}); in terms of the per-system winning
fraction against ML, mixing improves on the base estimator in four settings,
ties in one (TC-SURE, $68\%$), and is lower in one (SS-SURE, $49\%$ versus
$42\%$), the setting where the base estimator is already statistically
indistinguishable from ML.  Both fractions are reported in
Table~\ref{tab:tail} so this statement can be checked directly.  The FIT
values in Table~\ref{tab:tail-fit} tell the same story in the relative
metric.  Table~\ref{tab:tail-quantiles} shows
that the plug-in rule assigns smaller weights to the EB estimator under the
more vulnerable EB hyperparameter choice (median $0.37$ for TC, $0.11$ for
SS), while SURE and GCV receive larger weights because their base estimates
are already more robust.
Table~\ref{tab:tail-xmse} connects the empirical risk gaps to the XMSE
criterion: for the base estimators, the sign of the oracle XMSE agrees
with the sign of the observed sample MSE gap in 88--97\% of systems.  After
mixing, the average sample gap is negative in all six settings, although the
plug-in XMSE sign is naturally noisier because it also contains the random
weight-selection step.

Table~\ref{tab:tail-baselines} adds the natural competitors.  The SURE-tuned
weight, which minimizes an unbiased estimate of the finite-sample risk
difference for the frozen hyperparameter, performs essentially identically to
the XMSE plug-in weight under mismatch (within $0.003$ in every row) while
choosing systematically larger weights; its advantage appears in favorable
settings (Section~\ref{sec:semireal}).  The hard $\{0,1\}$ rule is slightly
worse than the soft mixture in the four rows where the two differ at the
displayed precision and ties it in the remaining two (SS-SURE and SS-GCV),
which quantifies the small gain of weight interpolation over pure ML/EB
selection.  The trace-corrected variant coincides with the plug-in rule to the
displayed precision in every row except TC-GCV, where the two differ by
$0.001$, because with known $\sigma^2$, $n=20$, and $N=50$ the trace
terms are small relative to the tail-mismatch quadratic forms.  Concerning
the projection of the weight (Remark~\ref{rem:projection}), the unconstrained
ratio fell below $0$ in $8$--$31\%$ and exceeded $1$ in $0$--$34\%$ of
realizations across the six settings, so the projection is active for a
nontrivial minority of realizations.

\subsection{Sample-size and SNR sensitivity}

To check whether the plug-in mixed estimator approaches the oracle mixed
estimator as the sample size increases, we run a sample-size sweep for the TC
tail-mismatch setting with scaled EB hyperparameter selection, using 80 systems
and 300 Monte Carlo repetitions per system.  Table~\ref{tab:n-sweep} shows that
the plug-in mixed estimator moves closer to the oracle mixed estimator as $N$
increases, matching it to three digits from $N=100$ on.  At $N=30$ the noise
level is high enough that shrinkage helps on average; the plug-in mixed
estimator improves over both ML and the base estimator, but remains visibly
away from the oracle ($4.79$ versus $4.31$), which reflects the noisier
finite-sample XMSE approximation at small $N$.

\begin{table*}[t]
\centering
\begin{tabular}{lcccccc}
\hline
$N$ & ML & base & oracle mix & plug-in mix & $\hat\omega$ & mix $<$ ML \\
\hline
30 & 5.128 & 5.044 & 4.310 & 4.792 & 0.418 & 70\% \\
50 & 0.918 & 1.010 & 0.900 & 0.908 & 0.368 & 69\% \\
70 & 0.485 & 0.509 & 0.479 & 0.480 & 0.344 & 71\% \\
100 & 0.282 & 0.286 & 0.281 & 0.281 & 0.376 & 84\% \\
150 & 0.162 & 0.163 & 0.161 & 0.161 & 0.344 & 89\% \\
\hline
\end{tabular}

\caption{Sample-size sensitivity under TC tail mismatch, scaled EB, SNR 10.}
\label{tab:n-sweep}
\end{table*}

We also sweep SNR at $N=50$, again using 80 systems and 300 Monte Carlo
repetitions per system.  Table~\ref{tab:snr-sweep} shows that the plug-in mixed
estimator improves over ML across all tested SNR levels.  At very low SNR the
base EB estimator is already helpful, while at moderate and high SNR the
XMSE-aware mixture provides protection against the tail-mismatch penalty.

\begin{table*}[t]
\centering
\begin{tabular}{lcccccc}
\hline
SNR & ML & base & oracle mix & plug-in mix & $\hat\omega$ & mix $<$ ML \\
\hline
1 & 0.854 & 0.762 & 0.729 & 0.766 & 0.394 & 98\% \\
3 & 0.857 & 0.881 & 0.792 & 0.811 & 0.426 & 86\% \\
10 & 0.842 & 0.879 & 0.816 & 0.824 & 0.438 & 72\% \\
30 & 0.854 & 0.875 & 0.845 & 0.848 & 0.413 & 62\% \\
100 & 0.853 & 0.859 & 0.850 & 0.851 & 0.458 & 69\% \\
\hline
\end{tabular}

\caption{SNR sensitivity under TC tail mismatch, scaled EB, $N=50$.}
\label{tab:snr-sweep}
\end{table*}

Figure~\ref{fig:sensitivity} visualizes the same sensitivity experiments.
Both panels plot the mean MSE of each estimator \emph{relative to ML}, so the
ML reference is the horizontal line at one and the curves remain separated at
every $N$ and SNR; absolute MSE values are in
Tables~\ref{tab:n-sweep}--\ref{tab:snr-sweep}.  The sample-size panel shows
the plug-in mixed estimator tracking the oracle mixed estimator as $N$ grows,
while the SNR panel shows that the mixed estimator stays at or below the ML
line across the tested signal-to-noise ratios.

\begin{figure*}[t]
\centering
\includegraphics[width=\textwidth]{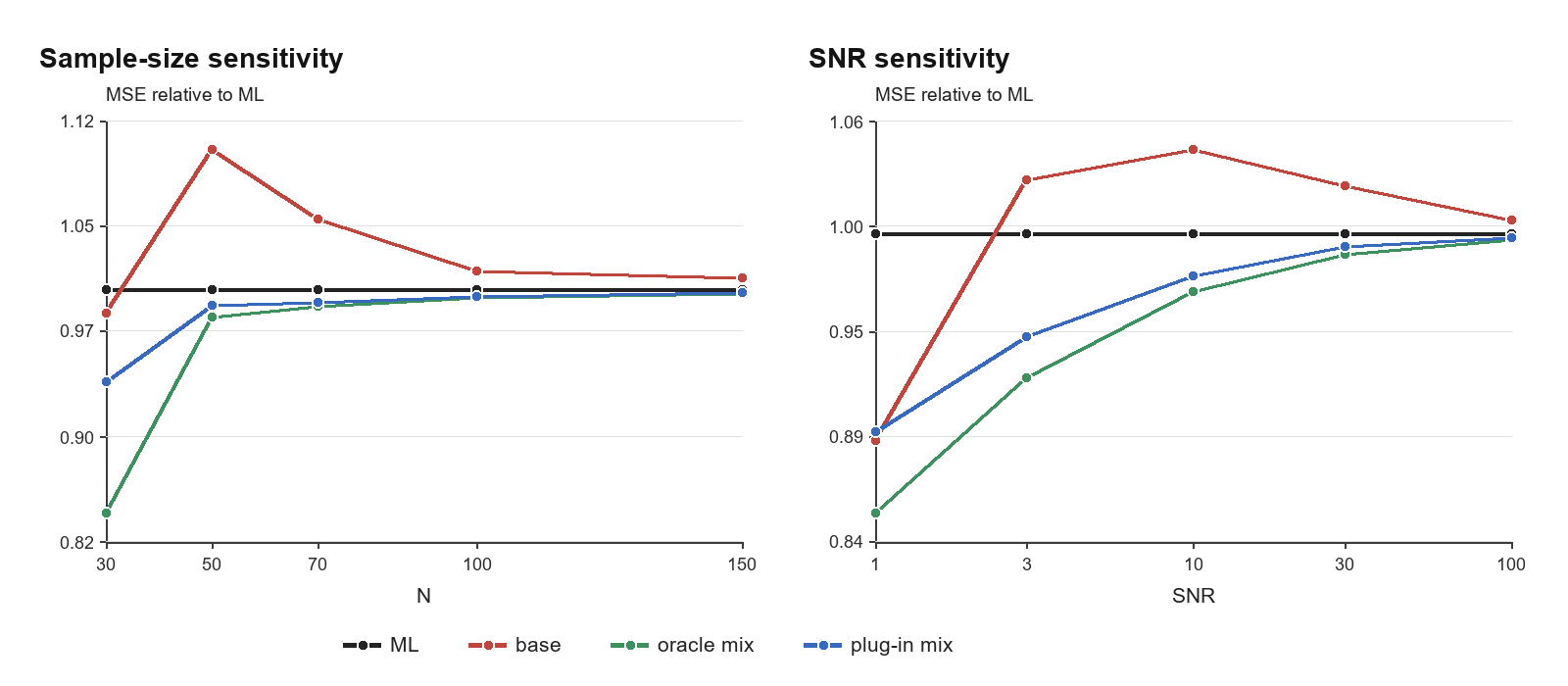}
\caption{Sample-size and SNR sensitivity under TC tail mismatch with scaled EB
hyperparameter selection.  Both panels show mean MSE relative to ML.}
\label{fig:sensitivity}
\end{figure*}

\subsection{Finite-candidate kernel selection}

We next evaluate the finite-candidate kernel-selection rule using the
candidate set $\{\mathrm{RI},\mathrm{TC},\mathrm{SS}\}$ over 100 systems and
300 Monte Carlo repetitions per system.  The protocol is the one a user
faces: for \emph{each individual noise realization}, the rule computes the
single-realization plug-in criterion $\widehat{q}_j$ for every candidate and
selects the minimizing kernel, and the selected mixed estimator is scored on
that same realization.  This matches
Theorem~\ref{thm:finite-kernel-selection}, whose plug-in criterion is a
single-realization quantity; averaging $\widehat{q}_j$ over the Monte Carlo
repetitions before selecting would use information a single user does not
have.  The selection is
risk-oriented: it minimizes the plug-in XMSE criterion and is not intended to
recover the data-generating kernel.  As a standard-practice baseline, we also
report per-realization kernel selection by maximized marginal likelihood
(evidence), where for each candidate the scale $\eta$ is chosen by maximizing
the Gaussian evidence and the winning kernel's regularized estimator is
scored (column ``evid.\ MSE'').

Table~\ref{tab:kernel-selection} reports the realization-level selection
frequencies, the fraction of realizations where the selected kernel matches
the per-system best fixed candidate (``match best''), the realization-level
MSE of the selected estimator, the best fixed-candidate MSE, and the regret,
i.e., the mean excess of the selected MSE over the best fixed candidate.
The regret column is the empirical analogue of
Corollary~\ref{cor:finite-set-oracle-inequality}: the theory controls the oracle
criterion regret of the selected candidate by the uniform plug-in criterion
error, while the table reports the same risk-oriented comparison on realized
Monte Carlo test errors.  Thus a small regret is the target predicted by the
oracle inequality even when the selected kernel is not the data-generating
kernel.

\begin{table*}[t]
\centering
\small
\setlength{\tabcolsep}{3pt}
\begin{tabular}{lccccccccc}
\hline
Setting & RI & TC & SS & match best & selected MSE & best MSE & regret & evid.\ MSE & ML MSE \\
\hline
tc\_aligned & 68\% & 15\% & 17\% & 74\% & 0.768 & 0.755 & 0.013 & 0.746 & 0.824 \\
tc\_tail & 99\% & 1\% & 0\% & 90\% & 0.839 & 0.837 & 0.002 & 0.869 & 0.890 \\
\hline
\end{tabular}

\caption{Finite-candidate kernel selection with per-realization selection.
RI/TC/SS columns are realization-level selection frequencies of the plug-in
XMSE rule; ``evid.\ MSE'' is the per-realization marginal-likelihood
baseline, which selects RI in $99$--$100\%$ of realizations in both
settings.}
\label{tab:kernel-selection}
\end{table*}

In the tail-mismatch setting, the XMSE rule chooses RI in 99\% of
realizations, effectively avoiding the decaying TC/SS kernels that are poorly
aligned with the true parameter shape; the selected MSE ($0.839$) is within
$0.002$ of the best fixed candidate and improves over both ML ($0.890$) and
the evidence baseline ($0.869$), whose selected EB estimator inherits the
harmful shrinkage.  In the aligned setting, the realization-level selection is
noisier (it
distributes $68/15/17\%$ across RI/TC/SS and matches the best fixed candidate
in $74\%$ of realizations), but the regret remains small ($0.013$), which is
the risk-oriented target.  In this favorable setting the evidence baseline is
slightly better than the XMSE rule ($0.746$ versus $0.768$): when the kernel
family is well aligned, standard evidence maximization combined with full EB
shrinkage is hard to beat, and the value of the XMSE rule lies in the
mismatched regime.

\subsection{Semi-real fixed impulse-response benchmark}
\label{sec:semireal}

As a complementary check, we use a fixed stable oscillatory impulse response
rather than drawing a new parameter vector for each system.  The benchmark uses
$n=30$, $N=80$, SNR 10, 80 input realizations, and 300 Monte Carlo repetitions
per input; TC and SS rows share the same seed and therefore the same inputs
and noise (paired).  The impulse response is normalized and then rescaled to
the desired SNR for each input realization.  Table~\ref{tab:semireal} reports
the results for TC and SS kernels with EB, SURE-type, and finite-sample GCV
scale selection, and Table~\ref{tab:semireal-baselines} the corresponding
baseline comparison.

\begin{table*}[t]
\centering
\small
\begin{tabular}{lccccccc}
\hline
Setting & ML & base & oracle mix & plug-in mix & $\hat\omega$ & base $<$ ML & mix $<$ ML \\
\hline
TC-EB-semireal & 0.766 & 0.443 & 0.443 & 0.546 & 0.711 & 100\% & 100\% \\
TC-SURE-semireal & 0.766 & 0.474 & 0.488 & 0.565 & 0.631 & 100\% & 100\% \\
TC-GCV-semireal & 0.766 & 0.400 & 0.409 & 0.674 & 0.263 & 100\% & 100\% \\
SS-EB-semireal & 0.766 & 0.451 & 0.451 & 0.699 & 0.169 & 100\% & 100\% \\
SS-SURE-semireal & 0.766 & 0.586 & 0.601 & 0.677 & 0.427 & 100\% & 100\% \\
SS-GCV-semireal & 0.766 & 0.313 & 0.334 & 0.765 & 0.004 & 100\% & 99\% \\
\hline
\end{tabular}

\caption{Semi-real fixed impulse-response benchmark, $N=80$, $n=30$, SNR 10,
known $\sigma^2$.}
\label{tab:semireal}
\end{table*}

\begin{table*}[t]
\centering
\small
\begin{tabular}{lcccccc}
\hline
Setting & ML & base & plug-in mix & SURE-$\omega$ mix & hard $\{0,1\}$ & corrected mix \\
\hline
TC-EB-semireal & 0.766 & 0.443 & 0.546 & 0.448 & 0.569 & 0.495 \\
TC-SURE-semireal & 0.766 & 0.474 & 0.565 & 0.474 & 0.563 & 0.544 \\
TC-GCV-semireal & 0.766 & 0.400 & 0.674 & 0.414 & 0.719 & 0.487 \\
SS-EB-semireal & 0.766 & 0.451 & 0.699 & 0.463 & 0.755 & 0.567 \\
SS-SURE-semireal & 0.766 & 0.586 & 0.677 & 0.586 & 0.693 & 0.631 \\
SS-GCV-semireal & 0.766 & 0.313 & 0.765 & 0.333 & 0.765 & 0.438 \\
\hline
\end{tabular}

\caption{Baseline comparison on the semi-real benchmark (mean MSE); columns
as in Table~\ref{tab:tail-baselines}.}
\label{tab:semireal-baselines}
\end{table*}

In this benchmark the decaying kernels are well matched to the stable impulse
response, so the base regularized estimator is very strong (mean MSE
$0.31$--$0.59$ against $0.77$ for ML).  The plug-in mixed estimator is never
worse than ML, but it is \emph{overly conservative} here: it retains only part
of the achievable gain, most visibly in the GCV and SS rows (e.g., SS-GCV
keeps almost none of it, with mean weight $0.004$).  The mechanism is the
finite-sample inflation of the plug-in squared-bias component discussed in
Remark~\ref{rem:corrected}: quadratic forms in $\thetaml$ involving the
ill-conditioned $Q=K_{\mathrm{SS}}^{-1}$ are dominated by their trace bias,
which makes shrinkage look more harmful than it is.  Consistent with this
diagnosis, the trace-corrected variant recovers a large part of the gap
(TC-GCV $0.674\to0.487$, SS-GCV $0.765\to0.438$), and the SURE-tuned weight,
which estimates the finite-sample risk difference directly, recovers nearly
all of it (within $0.020$ of the base in every row).  The unconstrained
oracle ratio exceeds one in these settings (mean values between $1.1$ and
$27$), i.e., over-shrinkage beyond EB would be profitable, so the projected
oracle weight sits at or near the upper boundary.

Two readings of Table~\ref{tab:semireal} deserve comment.  First, the oracle
mixed estimator is slightly \emph{worse} than the base estimator in the
SURE/GCV rows (e.g., $0.488$ versus $0.474$ for TC-SURE).  This is not a
contradiction of Theorem~\ref{thm:oracle-weight}: the oracle weight minimizes
the limiting XMSE criterion, and its dominance over the base estimator holds
at the XMSE scale asymptotically, not at every finite $N$; in strongly
favorable settings the finite-sample optimal weight can even exceed one,
which the criterion, restricted to $[0,1]$, cannot express.  Second, the
favorable-case conservativeness shown here is the price of the protection
demonstrated in the tail-mismatch experiment; the baseline table makes the
trade-off explicit rather than hiding it.  This supports the intended
interpretation of the method as a protection mechanism with a quantified
cost: it retreats toward ML under harmful regularization, and in favorable
settings the trace-corrected or SURE-weighted variants recover most of the
remaining benefit.

The next two public benchmarks are deliberately interpreted differently from
the synthetic experiments.  They use the same linear FIR estimator on nonlinear
input-output systems, so they are out-of-model robustness checks rather than
tests of state-of-the-art nonlinear identification accuracy.  Two protocol
details apply to both benchmarks and differ from the synthetic experiments.
First, the noise variance is not known: it is estimated from the ML residuals
as $\widehat{\sigma}^2=\|y-\Phi\thetaml\|_2^2/(N-n)$ on the training segment,
and this estimate enters the regularized estimator, the plug-in components,
and the trace correction.  Because the linear FIR model is misspecified for
these nonlinear systems, $\widehat{\sigma}^2$ absorbs unmodeled dynamics in
addition to measurement noise.  Second, the FIR regression matrix is built
from the zero-initialized input sequence, so the first $n$ training rows
contain the initial transient; the test-side predictions discard a 50-sample
initialization window, while the training side does not.  Both choices affect
all estimators identically, and we disclose them for reproducibility.  In the
tables, the mixing weight is a useful diagnostic: weights near zero indicate
that the plug-in XMSE criterion detects harmful shrinkage and returns ML,
whereas weights near one in rows where EB is worse than ML indicate
finite-sample calibration error in the estimated XMSE components.  The
trace-corrected weight $\widehat{\omega}_{\mathrm{c}}$ and the corresponding
RMSE/FIT columns test whether that calibration error is explained by the
trace bias of Remark~\ref{rem:corrected}.

\subsection{Public Silverbox benchmark}

As an archival real-data check, we use the Silverbox benchmark
\cite{wigren2013silverbox}, an electronic Duffing-oscillator system with an
official train/test split.  Because the benchmark is nonlinear, the purpose is
not to obtain a state-of-the-art Silverbox model with a linear FIR estimator.
Instead, we ask whether the XMSE-aware rule retreats from harmful shrinkage,
and where its plug-in calibration limits appear, on a public input-output
dataset outside the synthetic data-generating assumptions.  We fit a
length-$50$ FIR model using the first 500 samples of the official training
record and evaluate one-step FIR predictions on the three official test
records (\texttt{multisine}, \texttt{arrow\_full}, and
\texttt{arrow\_no\_extrap}) after a 50-sample initialization window.  Inputs
and outputs are standardized using the training record.

\begin{table*}[t]
\centering
\scriptsize
\resizebox{\textwidth}{!}{\begin{tabular}{l l l r r r r r r r r r r}
\hline
test & kernel & hpe & $\widehat{\omega}$ & $\widehat{\omega}_{\mathrm{c}}$ & RMSE ML & RMSE EB & RMSE mix & RMSE corr & FIT ML & FIT EB & FIT mix & FIT corr \\
\hline
multisine & TC & EB & 0.219 & 1.000 & 0.2198 & 0.2230 & 0.2203 & 0.2230 & 77.5 & 77.1 & 77.4 & 77.1 \\
multisine & TC & SURE & 0.945 & 1.000 & 0.2198 & 0.2216 & 0.2215 & 0.2216 & 77.5 & 77.3 & 77.3 & 77.3 \\
multisine & TC & GCV & 0.000 & 1.000 & 0.2198 & 0.2234 & 0.2198 & 0.2234 & 77.5 & 77.1 & 77.5 & 77.1 \\
multisine & SS & EB & 0.044 & 1.000 & 0.2198 & 0.2229 & 0.2199 & 0.2229 & 77.5 & 77.2 & 77.5 & 77.2 \\
multisine & SS & SURE & 0.352 & 1.000 & 0.2198 & 0.2208 & 0.2202 & 0.2208 & 77.5 & 77.4 & 77.4 & 77.4 \\
multisine & SS & GCV & 0.000 & 1.000 & 0.2198 & 0.2257 & 0.2198 & 0.2257 & 77.5 & 76.9 & 77.5 & 76.9 \\
arrow\_full & TC & EB & 0.219 & 1.000 & 0.2955 & 0.2968 & 0.2954 & 0.2968 & 69.3 & 69.1 & 69.3 & 69.1 \\
arrow\_full & TC & SURE & 0.945 & 1.000 & 0.2955 & 0.2963 & 0.2962 & 0.2963 & 69.3 & 69.2 & 69.2 & 69.2 \\
arrow\_full & TC & GCV & 0.000 & 1.000 & 0.2955 & 0.2952 & 0.2955 & 0.2952 & 69.3 & 69.3 & 69.3 & 69.3 \\
arrow\_full & SS & EB & 0.044 & 1.000 & 0.2955 & 0.2971 & 0.2955 & 0.2971 & 69.3 & 69.1 & 69.3 & 69.1 \\
arrow\_full & SS & SURE & 0.352 & 1.000 & 0.2955 & 0.2961 & 0.2956 & 0.2961 & 69.3 & 69.2 & 69.2 & 69.2 \\
arrow\_full & SS & GCV & 0.000 & 1.000 & 0.2955 & 0.2972 & 0.2955 & 0.2972 & 69.3 & 69.1 & 69.3 & 69.1 \\
arrow\_no\_extrap & TC & EB & 0.219 & 1.000 & 0.1815 & 0.1828 & 0.1814 & 0.1828 & 76.5 & 76.4 & 76.5 & 76.4 \\
arrow\_no\_extrap & TC & SURE & 0.945 & 1.000 & 0.1815 & 0.1822 & 0.1822 & 0.1822 & 76.5 & 76.4 & 76.4 & 76.4 \\
arrow\_no\_extrap & TC & GCV & 0.000 & 1.000 & 0.1815 & 0.1826 & 0.1815 & 0.1826 & 76.5 & 76.4 & 76.5 & 76.4 \\
arrow\_no\_extrap & SS & EB & 0.044 & 1.000 & 0.1815 & 0.1830 & 0.1815 & 0.1830 & 76.5 & 76.3 & 76.5 & 76.3 \\
arrow\_no\_extrap & SS & SURE & 0.352 & 1.000 & 0.1815 & 0.1821 & 0.1817 & 0.1821 & 76.5 & 76.4 & 76.5 & 76.4 \\
arrow\_no\_extrap & SS & GCV & 0.000 & 1.000 & 0.1815 & 0.1842 & 0.1815 & 0.1842 & 76.5 & 76.2 & 76.5 & 76.2 \\
\hline
\end{tabular}
}
\caption{Public Silverbox benchmark with a length-$50$ FIR model trained on
the first 500 samples of the official training record.  RMSE and
$\mathrm{FIT}_y$ are computed on standardized official test records after a
50-sample initialization window.  The noise variance is estimated from ML
training residuals.  $\widehat{\omega}$ is the plug-in weight and
$\widehat{\omega}_{\mathrm{c}}$, RMSE corr, FIT corr the trace-corrected
variant.}
\label{tab:silverbox}
\end{table*}

Table~\ref{tab:silverbox} gives both a conservative-success case and a visible
calibration failure mode.  The weight is selected once on the training data
and reused across the three test records.  When the training-side plug-in
criterion flags the finite-sample GCV regularized fit as harmful, the
XMSE-aware weight is zero and the mixed estimator returns ML exactly on every
test record.  In the EB rows, the rule uses a small amount of shrinkage and is
indistinguishable from ML at the displayed precision on the arrow tests
(differences in the fourth decimal of RMSE on a single test record, for which
no uncertainty quantification is available), while avoiding the larger
degradation of the base EB estimator.  We therefore do not claim improvement
over ML on this benchmark; the supported claim is that the safeguard avoids
most of the base estimator's degradation.  The SURE-type rows are the failure
mode: the plug-in criterion assigns a large weight (e.g., $0.945$ for TC) and
inherits most of the base estimator's error.  The trace-corrected variant
does not repair this; instead, $\widehat{\omega}_{\mathrm{c}}=1$ in every
row, because the misspecification-inflated $\widehat{\sigma}^2$ makes the
subtracted traces as large as the plug-in quadratic forms, so the corrected
squared-bias component collapses to its zero floor and the corrected rule
returns full EB.  The benchmark calibration failure is thus not explained by
the trace bias alone: under model misspecification, both the uncorrected and
the corrected plug-in components are operating outside the theory, and the
mixed rule degrades gracefully to one of its endpoints rather than failing
unboundedly.  This is the precise sense in which the benchmark is a
robustness check rather than evidence of improvement outside the linear FIR
setting.

\subsection{Public Cascaded Tanks benchmark}

As a second public input-output benchmark, we use the Cascaded Tanks dataset
\cite{schoukens2020cascaded}, accessed through the
official \texttt{nonlinear-benchmarks} dataloader.  The data consist of one
training record and one validation record of length 1024.  We fit the same
length-$50$ FIR model using the first 700 training samples and evaluate
one-step predictions on the validation record after a 50-sample initialization
window.  Inputs and outputs are standardized using the training segment.

\begin{table*}[t]
\centering
\scriptsize
\resizebox{\textwidth}{!}{\begin{tabular}{l l l r r r r r r r r r r}
\hline
test & kernel & hpe & $\widehat{\omega}$ & $\widehat{\omega}_{\mathrm{c}}$ & RMSE ML & RMSE EB & RMSE mix & RMSE corr & FIT ML & FIT EB & FIT mix & FIT corr \\
\hline
cascaded\_tanks & TC & EB & 0.490 & 1.000 & 0.4719 & 0.4766 & 0.4742 & 0.4766 & 54.7 & 54.3 & 54.5 & 54.3 \\
cascaded\_tanks & TC & SURE & 1.000 & 1.000 & 0.4719 & 0.4751 & 0.4751 & 0.4751 & 54.7 & 54.4 & 54.4 & 54.4 \\
cascaded\_tanks & TC & GCV & 0.000 & 1.000 & 0.4719 & 0.4823 & 0.4719 & 0.4823 & 54.7 & 53.8 & 54.7 & 53.8 \\
cascaded\_tanks & SS & EB & 0.111 & 1.000 & 0.4719 & 0.4743 & 0.4722 & 0.4743 & 54.7 & 54.5 & 54.7 & 54.5 \\
cascaded\_tanks & SS & SURE & 0.558 & 1.000 & 0.4719 & 0.4729 & 0.4725 & 0.4729 & 54.7 & 54.6 & 54.7 & 54.6 \\
cascaded\_tanks & SS & GCV & 0.000 & 1.000 & 0.4719 & 0.4831 & 0.4719 & 0.4831 & 54.7 & 53.7 & 54.7 & 53.7 \\
\hline
\end{tabular}
}
\caption{Public Cascaded Tanks benchmark with a length-$50$ FIR model trained
on the first 700 samples of the training record.  RMSE and $\mathrm{FIT}_y$
are computed on standardized validation data after a 50-sample initialization
window.  Noise variance estimated from ML training residuals; corrected
columns as in Table~\ref{tab:silverbox}.}
\label{tab:cascaded-tanks}
\end{table*}

Table~\ref{tab:cascaded-tanks} shows the same split between protective behavior
and calibration limits.  The regularized EB estimates are slightly worse than
ML in this linear FIR evaluation, and the mixed rule either retreats to ML, as
in the GCV rows, or uses only a partial shrinkage weight.  The SURE-type rows
remain the limitation: the plug-in criterion assigns full EB weight for the TC
kernel and therefore inherits the base estimator's small degradation
($0.4751$ versus $0.4719$ for ML).  As on Silverbox, the trace-corrected
variant sets $\widehat{\omega}_{\mathrm{c}}=1$ throughout and cannot repair
this, for the same $\widehat{\sigma}^2$-inflation reason, which localizes the
failure in the misspecification-driven calibration of the plug-in components
rather than in the oracle mixing argument or in the trace bias alone.  The
worst case across all rows is bounded by the base estimator's own error,
which is the graceful-degradation property of the convex mixture.
Thus the benchmark strengthens the robustness evidence without claiming that
the XMSE plug-in rule is uniformly optimal on nonlinear real-data problems.

\section{Limitations and Future Work}

The present theory separates the fixed-weight XMSE expansion from the plug-in
selection of the mixing weight.  We prove consistency of the plug-in weight and
derive a second-order oracle regret rate for the corresponding design
criterion.  The transfer result (Theorem~\ref{thm:plugin-xmse-transfer})
carries this regret rate to the deterministic fixed-weight risk curve
evaluated at the selected weight, and its uniform-convergence ingredient is
proved from the inherited expansion rather than assumed.  The remaining gap
is explicit: the XMSE of the data-dependent estimator
$\thetamix(\widehat{\omega}_{\mathrm{XMSE}})$ itself, with the random weight
inside the expectation, differs from the evaluated risk curve by cross-moments
between the weight and the EB correction, which are generated by their common
dependence on the data (Remark~\ref{rem:transfer-scope}).  Bounding these
cross-moments---plausibly by a delta-method expansion of the smooth
scaled-kernel weight map away from the projection boundary, with extra care
near the projection and threshold boundaries where the weight map is
nonsmooth---is the main open theoretical step, and none of our experimental
claims rely on it.  We also provide a
thresholded decision result at the zero-bias boundary and a finite-set oracle
inequality for kernel selection, with an extension to growing candidate
dictionaries under uniform plug-in approximation and a high-probability
oracle bound, including a sub-Gaussian rate specialization, and a
compact multi-parameter plug-in consistency result based on uniform convergence
and an argmin argument.

Two design choices deserve a limitation note of their own.  First, the
projection of the weight onto $[0,1]$ deliberately forgoes the strictly
negative criterion values that an extrapolated weight could attain when
$C_{\mathrm{EB}}>0$ or $C_{\mathrm{EB}}<-2B_{\mathrm{EB}}$
(Remark~\ref{rem:projection}); the experiments quantify how often the
unconstrained ratio leaves $[0,1]$, but we do not study extrapolated rules.
Second, the plug-in components inherit an $O(\sigma^2/N)$ trace bias
(Remark~\ref{rem:corrected}).  The trace-corrected variant removes the
identified bias terms and repairs the over-conservative weights in the
well-specified semi-real setting, but it does not remove the variance of the
plug-in quadratic forms, and under model misspecification, where
$\widehat{\sigma}^2$ also absorbs model error, it over-corrects and returns
full EB weight on the public benchmarks.  Finite-sample and
misspecification-robust calibration of the plug-in components therefore
remains the main practical limitation.

All asymptotic statements are for fixed parameter dimension.  The growing
dictionary results allow the candidate menu to expand, but they do not by
themselves prove XMSE validity when the FIR order grows with $N$.  A
growing-order extension would need uniform control of the design eigenvalues,
kernel inverses, scale-estimation objectives, and plug-in component errors as
both the sample size and parameter dimension vary.  We view this as a separate
high-dimensional extension rather than an assumption hidden in the present
fixed-order theory.

The simulations focus on scaled kernels $P(\eta)=\eta K$ and on controlled
tail-mismatch settings that make kernel misalignment explicit, supplemented by
a semi-real fixed impulse-response benchmark and limited-training Silverbox and
Cascaded Tanks benchmarks.  This provides a clean stress test, a fixed-response
check, and two public real-data benchmarks, with SURE-tuned, hard-selection,
trace-corrected, and marginal-likelihood baselines.  One natural competitor is
not yet included: the closed-form biased estimators of \cite{ju2025biased},
which match the XMSE of an EB estimator without hyperparameter estimation and
would make an informative third anchor besides ML and EB.  A broader
experimental study should
include this comparison, additional input designs, unknown noise variance in
the synthetic settings, non-Gaussian
disturbances, more archival system-identification benchmarks, and model classes
beyond linear FIR when the data are strongly nonlinear.  Although we provide
finite-set, growing-dictionary, and compact multi-parameter consistency
guarantees, plus a 100-system numerical validation, a larger automatic
kernel-selection study across additional kernels and operating conditions is
left for future work.

Finally, the finite-sample plug-in approximation inherits the limitations of
second-order asymptotic XMSE.  For very small samples, higher-order terms can
be visible, as seen in the $N=30$ sensitivity experiment.  Bootstrap or
higher-order corrections may improve the plug-in rule in this regime.

\section{Conclusion}

XMSE is not only a diagnostic tool for explaining empirical Bayes failures.  It
can also serve as a constructive design criterion for adaptive estimators.  The
mixed estimator studied here has a scalar oracle rule with second-order
dominance over both ML and the base EB estimator at the XMSE scale.  Its
plug-in version is supported by consistency, a boundary-robust oracle
inequality, an oracle-regret rate, and a transfer of that rate to the
fixed-weight risk curve evaluated at the selected weight, and
the primitive plug-in route extends to compact multi-parameter kernel families.
The same criterion also extends naturally to finite and growing candidate
dictionaries with high-probability and sub-Gaussian oracle guarantees.  The
result is a simple safeguard: it preserves useful regularization when the
kernel is appropriate and retreats toward ML when the regularizer is harmful.

\sloppy


\begin{thebibliography}{99}

\bibitem{efron1972}
B. Efron and C. Morris.
Limiting the risk of Bayes and empirical Bayes estimators---Part II: The
empirical Bayes case.
\emph{Journal of the American Statistical Association}, 67(337):130--139,
1972.

\bibitem{maritz2018}
J. S. Maritz and T. Lwin.
\emph{Empirical Bayes Methods with Applications}.
Chapman and Hall/CRC, 2018.

\bibitem{pillonetto2022}
G. Pillonetto, T. Chen, A. Chiuso, G. De Nicolao, and L. Ljung.
\emph{Regularized System Identification: Learning Dynamic Models from Data}.
Springer Nature, 2022.

\bibitem{robbins1956}
H. Robbins.
An empirical Bayes approach to statistics.
In \emph{Proceedings of the Third Berkeley Symposium on Mathematical
Statistics and Probability}, volume 1, pages 157--163, 1956.

\bibitem{james1961}
W. James and C. Stein.
Estimation with quadratic loss.
In \emph{Proceedings of the Fourth Berkeley Symposium on Mathematical
Statistics and Probability}, volume 1, pages 361--379, 1961.

\bibitem{efron1973}
B. Efron and C. Morris.
Stein's estimation rule and its competitors: An empirical Bayes approach.
\emph{Journal of the American Statistical Association}, 68(341):117--130,
1973.

\bibitem{morris1983}
C. N. Morris.
Parametric empirical Bayes inference: Theory and applications.
\emph{Journal of the American Statistical Association}, 78(381):47--55,
1983.

\bibitem{petrone2014}
S. Petrone, S. Rizzelli, J. Rousseau, and C. Scricciolo.
Empirical Bayes methods in classical and Bayesian inference.
\emph{Metron}, 72(2):201--215, 2014.

\bibitem{stein1981}
C. M. Stein.
Estimation of the mean of a multivariate normal distribution.
\emph{The Annals of Statistics}, 9(6):1135--1151, 1981.

\bibitem{wahba1990}
G. Wahba.
\emph{Spline Models for Observational Data}.
SIAM, 1990.

\bibitem{rasmussen2006}
C. E. Rasmussen and C. K. I. Williams.
\emph{Gaussian Processes for Machine Learning}.
MIT Press, 2006.

\bibitem{ljung1999}
L. Ljung.
\emph{System Identification: Theory for the User}.
Prentice Hall, 2nd edition, 1999.

\bibitem{pillonetto2010}
G. Pillonetto and G. De Nicolao.
A new kernel-based approach for linear system identification.
\emph{Automatica}, 46(1):81--93, 2010.

\bibitem{chen2012}
T. Chen, H. Ohlsson, and L. Ljung.
On the estimation of transfer functions, regularizations and Gaussian
processes---Revisited.
\emph{Automatica}, 48(8):1525--1535, 2012.

\bibitem{pillonetto2014}
G. Pillonetto, F. Dinuzzo, T. Chen, G. De Nicolao, and L. Ljung.
Kernel methods in system identification, machine learning and function
estimation: A survey.
\emph{Automatica}, 50(3):657--682, 2014.

\bibitem{chen2014ssm}
T. Chen and L. Ljung.
Constructive state space model induced kernels for regularized system
identification.
\emph{IFAC Proceedings Volumes}, 47(3):1047--1052, 2014.

\bibitem{chen2014sparse}
T. Chen, M. S. Andersen, L. Ljung, A. Chiuso, and G. Pillonetto.
System identification via sparse multiple kernel-based regularization using
sequential convex optimization techniques.
\emph{IEEE Transactions on Automatic Control}, 59(11):2933--2945, 2014.

\bibitem{chen2015obf}
T. Chen and L. Ljung.
Regularized system identification using orthonormal basis functions.
In \emph{Proceedings of the European Control Conference}, pages 1291--1296,
2015.

\bibitem{pillonetto2016atomic}
G. Pillonetto, T. Chen, A. Chiuso, G. De Nicolao, and L. Ljung.
Regularized linear system identification using atomic, nuclear and
kernel-based norms: The role of the stability constraint.
\emph{Automatica}, 69:137--149, 2016.

\bibitem{carli2017maxent}
F. P. Carli, T. Chen, and L. Ljung.
Maximum entropy kernels for system identification.
\emph{IEEE Transactions on Automatic Control}, 62(3):1471--1477, 2017.

\bibitem{chen2018kernel}
T. Chen.
On kernel design for regularized LTI system identification.
\emph{Automatica}, 90:109--122, 2018.

\bibitem{chenmh2020}
M. Chen, Z. Xu, J. Zhao, C. Song, Y. Zhu, and Z. Shao.
Nonparametric identification based on multi-inherited Gaussian process
regression for batch process.
\emph{Industrial \& Engineering Chemistry Research}, 59(47):20757--20766, 2020.

\bibitem{chenmh2022}
M. Chen, Z. Xu, J. Zhao, Y. Zhu, and Z. Shao.
Nonparametric identification of batch process using two-dimensional
kernel-based Gaussian process regression.
\emph{Chemical Engineering Science}, 250:117372, 2022.

\bibitem{chiuso2016}
A. Chiuso.
Regularization and Bayesian learning in dynamical systems: Past, present and
future.
\emph{Annual Reviews in Control}, 41:24--38, 2016.

\bibitem{chen2013implementation}
T. Chen and L. Ljung.
Implementation of algorithms for tuning parameters in regularized least
squares problems in system identification.
\emph{Automatica}, 49(7):2213--2220, 2013.

\bibitem{pillonetto2015tuning}
G. Pillonetto and A. Chiuso.
Tuning complexity in regularized kernel-based regression and linear system
identification: The robustness of the marginal likelihood estimator.
\emph{Automatica}, 58:106--117, 2015.

\bibitem{mu2018}
B. Mu, T. Chen, and L. Ljung.
On asymptotic properties of hyperparameter estimators for kernel-based
regularization methods.
\emph{Automatica}, 94:381--395, 2018.

\bibitem{mu2018gcv}
B. Mu, T. Chen, and L. Ljung.
Asymptotic properties of generalized cross validation estimators for
regularized system identification.
\emph{IFAC-PapersOnLine}, 51(15):203--208, 2018.

\bibitem{ju2021gcv}
Y. Ju, T. Chen, B. Mu, and L. Ljung.
On asymptotic distribution of generalized cross validation hyper-parameter
estimator for regularized system identification.
In \emph{Proceedings of the 60th IEEE Conference on Decision and Control},
pages 1598--1602, 2021.

\bibitem{ju2022sure}
Y. Ju, T. Chen, B. Mu, and L. Ljung.
On convergence in distribution of Stein's unbiased risk hyper-parameter
estimator for regularized system identification.
In \emph{Proceedings of the 41st Chinese Control Conference}, pages
1491--1496, 2022.

\bibitem{mu2024cv}
B. Mu and T. Chen.
On asymptotic optimality of cross-validation estimators for kernel-based
regularized system identification.
\emph{IEEE Transactions on Automatic Control}, 69(7):4352--4367, 2024.

\bibitem{zhang2024family}
M. Zhang, T. Chen, and B. Mu.
A family of hyperparameter estimators linking EB and SURE for kernel-based
regularization methods.
\emph{IEEE Transactions on Automatic Control}, 69(12):8674--8689, 2024.

\bibitem{mu2024when}
B. Mu, L. Ljung, and T. Chen.
When cannot regularization improve the least squares estimate in the
kernel-based regularized system identification.
\emph{Automatica}, 160:111442, 2024.

\bibitem{ju2026xmse}
Y. Ju, T. Chen, B. Wahlberg, and H. Hjalmarsson.
Excess mean squared error of empirical Bayes estimators.
\emph{IEEE Transactions on Automatic Control}, 2026.
doi: 10.1109/TAC.2026.3685569.

\bibitem{ju2025biased}
Y. Ju, B. Wahlberg, and H. Hjalmarsson.
Bayes and biased estimators without hyper-parameter estimation: Comparable
performance to the empirical-Bayes-based regularized estimator.
arXiv preprint arXiv:2503.11854, 2025.

\bibitem{wigren2013silverbox}
T. Wigren and J. Schoukens.
Three free data sets for development and benchmarking in nonlinear system
identification.
In \emph{Proceedings of the European Control Conference}, pages 2933--2938,
2013.

\bibitem{schoukens2020cascaded}
M. Schoukens, P. Mattson, T. Wigren, and J.-P. No{\"e}l.
Cascaded tanks benchmark combining soft and hard nonlinearities.
4TU.ResearchData, Dataset, 2020.
doi: 10.4121/12960104.

\end{thebibliography}
\end{document}